\documentclass{article} 
\usepackage[dvipsnames]{xcolor}
\usepackage[main,final]{neurips_2025}
\usepackage{geometry}
\usepackage{times}
\usepackage[utf8]{inputenc} 
\usepackage[T1]{fontenc}    

\usepackage{algorithm}
\usepackage{algorithmic}
\usepackage[titlenumbered,ruled,noend,algo2e]{algorithm2e}

\SetCommentSty{mycommfont}
\SetEndCharOfAlgoLine{}

\usepackage{url}            
\usepackage{booktabs}       
\usepackage{enumitem}
\usepackage{natbib}
\setcitestyle{round}
\usepackage{amsfonts}
\usepackage{graphicx}
\usepackage{nicefrac}       
\usepackage{microtype}      
\usepackage{amsthm}
\usepackage{amsmath,amsfonts,bm}
\makeatletter
\newtheorem*{rep@theorem}{\rep@title}
\newcommand{\newreptheorem}[2]{%
\newenvironment{rep#1}[1]{%
 \def\rep@title{#2 \ref{##1}}%
 \begin{rep@theorem}}%
 {\end{rep@theorem}}}
\makeatother

\usepackage[framemethod=TikZ]{mdframed}
\mdfdefinestyle{MyFrame}{%
    linecolor=black,
    outerlinewidth=.3pt,
    roundcorner=5pt,
    innertopmargin=1pt, 
    innerbottommargin=1pt, 
    innerrightmargin=1pt,
    innerleftmargin=1pt,
    backgroundcolor=black!0!white}

\mdfdefinestyle{MyFrame2}{%
    linecolor=white,
    outerlinewidth=1pt,
    roundcorner=2pt,
    innertopmargin=\baselineskip,
    innerbottommargin=\baselineskip,
    innerrightmargin=10pt,
    innerleftmargin=10pt,
    backgroundcolor=black!3!white}

\newcommand{\RNum}[1]{\uppercase\expandafter{\romannumeral #1\relax}}

\newcommand{\dx}{\mathrm{d}x}

\newcommand{\cN}{\mathcal{N}}

\newcommand{\diff}{\mathrm{d}}

\newcommand{\vertiii}[1]{{\left\vert\kern-0.25ex\left\vert\kern-0.25ex\left\vert #1
    \right\vert\kern-0.25ex\right\vert\kern-0.25ex\right\vert}}
\newcommand{\vertiiii}[1]{{\vert\kern-0.25ex\vert\kern-0.25ex\vert #1
    \vert\kern-0.25ex\vert\kern-0.25ex\vert}}

\usepackage{mathtools}

\newcommand\norm[1]{\left\lVert#1\right\rVert} 

\DeclareMathOperator*{\Id}{\mathrm{Id}}

\newcommand{\xhdr}[1]{{\noindent\bfseries #1}.}
\newcommand{\cut}[1]{}

\newcommand{\removelatexerror}{\let\@latex@error\@gobble}

\def\1{\bm{1}}

\DeclareMathAlphabet{\mathsfit}{\encodingdefault}{\sfdefault}{m}{sl}
\SetMathAlphabet{\mathsfit}{bold}{\encodingdefault}{\sfdefault}{bx}{n}

\newcommand{\pdata}{p_{\rm{data}}}

\newcommand{\softmax}{\mathrm{softmax}}

\newcommand{\Var}{\mathrm{Var}}

\newcommand{\utot}{u^\star}
\newcommand{\utheta}{u_{\theta}}

\newcommand{\utothat}{\hat{u}^\star}
\newcommand{\utothatemp}{\hat{u}_M^\star}
\newcommand{\normin}[1]{ \lVert {#1} \rVert}

\usepackage{tikz}

\usepackage{todonotes}
\usepackage{amssymb,fge}
\usepackage{thm-restate}
\usepackage{wrapfig}
\usepackage{bbm}
\usepackage{mathrsfs}
\usepackage{soul}
\usepackage{array}
\usepackage{multirow}
\usepackage{caption}
\usepackage{subcaption}
\usepackage{derivative}
\usepackage{circledsteps}
\usepackage{amsbsy}
\usepackage{stmaryrd}

\usepackage{diagbox}
\usepackage{empheq}
\usepackage{bbm}

\newcommand{\data}[1]{x^{(#1)}}
\newcommand{\batch}[1]{b^{(#1)}}

\usepackage[acronym,nomain]{glossaries}
\newglossary[slg]{symbolslist}{syi}{syg}{Symbolslist}
\makeglossaries
\usepackage{pgfplots}

\usepackage[colorlinks=true, allcolors=blue]{hyperref}
\usepackage{url}
\usepackage{dsfont}

\newcommand{\indep}{\perp \!\!\! \perp}
\newcommand{\ie}{{\em i.e.,~}}
\newcommand{\eg}{{\em e.g.,~}}

\usepackage{enumitem}
\setitemize{noitemsep,topsep=0pt,parsep=0pt,partopsep=0pt}
\setlist{topsep=0pt, leftmargin=*}

\newlist{lemmaenum}{enumerate}{1} 
\setlist[lemmaenum]{
  label=\emph{\roman*)},
  ref=\thedefinition~\emph{\roman*)}}

\newlist{thmenum}{enumerate}{1} 
\setlist[thmenum]{label=\emph{\roman*)}, ref=\thetheorem~\emph{\roman*)}}

\newlist{propenum}{enumerate}{1} 
\setlist[propenum]{
  label=\emph{\roman*)},
  ref=\theproposition~\emph{\roman*)}
}

 \newcommand{\dz}{~\mathrm{d}z}
\def\pdata{p_{\mathrm{data}}}
\def\pnot{p_0}
\def\pdatahat{\hat{p}_{\mathrm{data}}}

\def \ucond{u^{\mathrm{cond}}}

\usepackage{cleveref}
\newtheorem{theorem}{Theorem}
\newtheorem*{theorem*}{Theorem}

\newtheorem{proposition}[theorem]{Proposition}
\newtheorem*{proposition*}{Proposition}

\newtheorem*{example*}{Example}

\newreptheorem{theorem}{Theorem}

\crefalias{propenumi}{proposition}
\crefalias{lemmaenumi}{lemma}
\crefalias{thmenumi}{theorem}
\Crefname{lemmaenumi}{lemma}{Lemmas}
\Crefname{lemma}{Lemma}{Lemmas}
\Crefname{assumption}{Assumption}{Assumptions}
\Crefname{proposition}{Proposition}{Propositions}

\definecolor{darkgreen}{rgb}{0.0, 0.4, 0.13}

\title{On the Closed-Form of Flow Matching: Generalization Does Not Arise from Target Stochasticity}

\graphicspath{{figures/}}

\date{}

\author{%
  Quentin Bertrand\textsuperscript{1}\textsuperscript{5}\thanks{Equal contribution. Correspondence: \texttt{quentin.bertrand@inria.fr}.}, \;
  Anne Gagneux\textsuperscript{2}\footnotemark[1], \;
  Mathurin Massias\textsuperscript{3}\footnotemark[1], \;
  Rémi Emonet\textsuperscript{1}\textsuperscript{4}\footnotemark[1]
  \\
  \textsuperscript{1}Université Jean Monnet Saint-Étienne, CNRS, Institut d’Optique Graduate School, \\ Inria, Laboratoire Hubert Curien UMR 5516, F-42023 Saint-Étienne, France \\
  \textsuperscript{2}ENS de Lyon, CNRS, Université Claude Bernard Lyon 1,  Inria, \\ LIP UMR 5668, 69342 Lyon Cedex 07, France \\
  \textsuperscript{3}Inria, ENS de Lyon, CNRS, Université Claude Bernard Lyon 1, \\ LIP UMR 5668, 69342 Lyon Cedex 07, France \\
  \textsuperscript{4}Institut Universitaire de France \\
  \textsuperscript{5} Mila - Quebec AI Institute \\
  Code: \url{https://github.com/generativemodels/closedformfm}
}

\begin{document}

\maketitle

\begin{abstract}

\looseness-1
Modern deep generative models can now produce high-quality synthetic samples that are often indistinguishable from real training data. A growing body of research aims to understand why recent methods, such as diffusion and flow matching techniques, generalize so effectively. Among the proposed explanations are the inductive biases of deep learning architectures and the stochastic nature of the conditional flow matching loss.
In this work, we rule out the noisy nature of the loss as a key factor driving generalization in flow matching.
First, we empirically show that in high-dimensional settings, the stochastic and closed-form versions of the flow matching loss yield nearly equivalent losses. Then, using state-of-the-art flow matching models on standard image datasets, we demonstrate that both variants achieve comparable statistical performance, with the surprising observation that using the closed-form can even improve performance.

\end{abstract}


\section{Introduction}
Recent deep generative models, such as diffusion \citep{sohlDickstein2015diffusion,Ho2020,song2020diffusion} and flow matching models \citep{lipman2023flow,albergo2023stochasticinterpolant,liu2023rectifiedflow}, have achieved remarkable success in synthesizing realistic data across a wide range of domains.
State-of-the-art diffusion and flow matching methods are now capable of producing multi-modal outputs that are virtually indistinguishable from human-generated content, including images \citep{StablediffXL}, audio \citep{Borsos2023}, video \citep{Villegas2022,videoworldsimulators2024}, and text \citep{Gong2022diffusiontext,Xu2025energytext}.

A central question in deep generative modeling concerns the generalization capabilities and underlying mechanisms of these models.
Generative models generalization remains a puzzling phenomenon, raising a number of challenging and unresolved questions: whether generative models truly generalize is still the subject of active debate.
On one hand, several studies \citep{Carlini2023,somepalli2023diffusion,Somepalli2023understanding,Dar2023investigating} have shown that large diffusion models are capable of memorizing individual samples from the training set, including licensed photographs, trademarked logos, and sensitive medical data.

On the other hand, \citet{kadkhodaie2024generalization} have empirically demonstrated that while memorization can occur in low-data regimes, diffusion models trained on a \emph{sufficiently large} dataset exhibit clear signs of generalization. Taken together, recent work points to a sharp phase transition between memorization and generalization \citep{yoon2023diffusion,zhang2024emergencerepro}.
Multiple theories have been proposed to explain the puzzling generalization of diffusion and flow matching models.
On the one hand, \citet{kadkhodaie2024generalization,kamb2024analytic,ross2025memorization} suggested a geometric framework to understand the inductive bias of modern deep convolutional networks on images.
On the other hand, \citet{Vastola2025generalization} suggested that generalization is due to the \emph{noisy} nature of the training loss.
In this work, we clearly answer the following question:

\begin{center}
  \textit{Does training on noisy/stochastic targets improve flow matching generalization? \\
If not, what are the main sources of generalization?}
\end{center}

\xhdr{Contributions}
\begin{itemize}
    \item We challenge the prevailing belief that generalization in flow matching stems from an inherently noisy loss (\Cref{sub:target_sto}). This assumption, largely supported by studies in low-dimensional settings, fails to hold in realistic high-dimensional data regimes.

    \item Instead, we observe that generalization in flow matching emerges precisely when the limited-capacity neural network fails to approximate the \emph{optimal closed-form velocity field} (\Cref{sub:phasesapprox}).

    \item We identify two critical time intervals, at early and late times, where \emph{neural networks fail to approximate the optimal velocity field} (\Cref{sub:gentimes}). We show that generalization arises predominantly early along flow matching trajectories, aligning with the transition from the stochastic to the deterministic regime of the flow matching objective.

    \item Finally, on standard image datasets (CIFAR-10 and CelebA), we show that explicitly regressing against the optimal closed-form velocity field does not impair generalization and can, in some cases, enhance it (\Cref{sec:learning}).
\end{itemize}

The manuscript is organized as follows. \Cref{sec:cfmrecalls} reviews the fundamentals of conditional flow matching and recalls the closed-form of the ``optimal'' velocity field. Leveraging the closed-form expression of the flow matching velocity field, \Cref{sec:investigategeneralization} investigates the key sources of generalization in flow matching. In \Cref{sec:learning}, we introduce a learning algorithm based on the closed-form formula. Related work is discussed in detail in \Cref{sec:relatedwork}.

\section{Recalls on conditional flow matching}
\label{sec:cfmrecalls}

Let $\pnot = \cN(0, \Id)$ be the source distribution\footnote{the choice $p_0= \cN(0, \Id)$ is made for simplicity; more generic choices are possible and the reader can refer to \citet{Lipman2024guide,Gagneux2025visual_flow_matching,diffusion_fm_blogpost} for deeper introductions to flow matching.} and $\pdata$ the data distribution.
We are given $n$ data points $\data{1}, \dots, \data{n} \sim \pdata$, $\data{i} \in \mathbb{R}^d$.
The goal of flow matching is to find a velocity field $u: \mathbb{R}^d \times [0, 1] \to \mathbb{R}^d$, such that, if one solves on $[0,1]$ the ordinary differential equation
\begin{equation}\label{eq:ode}
    \begin{cases}
      x(0) = x_0 \in \mathbb{R}^d \\
      \dot x(t) = u(x(t), t)
    \end{cases}
\end{equation}
\looseness-1
then the law of $x(1)$ when $x_0\sim p_0$ is $\pdata$: one says that $u$ \emph{transports} $p_0$ to $\pdata$.
For every value of $t$ between $0$ and $1$, the law of $x(t)$ defines a \emph{probability path}, denoted $p(\cdot |t)$ that progressively transforms $p_0$ to $p_{\mathrm{data}}$.
If one knows the velocity field $u$, new samples can then be generated by sampling $x_0$ from $p_0$, solving the ordinary differential equation, and using $x(1)$ as the generated point.

In conditional flow matching, finding such a velocity field $u$ is achieved in the following way.
\begin{enumerate}[label=(\roman*)]
    \item First, define a conditioning variable $z$ independent of $t$, \eg $z = x_1 \sim \pdata$, \label{cond_item:1}
    \item Then, chose a conditional probability path $p(\cdot | z, t)$, \eg $p(\cdot | z=x_1, t) = \cN(t x_1, (1 -t)^2 \Id)$. \label{cond_item:2}
\end{enumerate}
Through the continuity equation~\citep[Sec. 3.5]{Lipman2024guide}, the choice~\ref{cond_item:2} of the conditional probability path $p(\cdot |z,t)$ defines a conditional velocity field $\ucond (x, z, t)$.
With the choices~\ref{cond_item:1} and~\ref{cond_item:2}, the conditional velocity field writes
\begin{align}
    \ucond(x, z=x_1, t) = \frac{x_1 - x}{1-t}
    \enspace.
\end{align}
The choice~\ref{cond_item:2} of the conditional probability paths $p(\cdot | z=x_1, t)$ fully defines a probability path $p(\cdot|t)$ (by marginalization against $z$) and thus defines an \emph{optimal velocity field} $\utot$ (through the continuity equation), that transports $p_0$ to $\pdata$ \citep[Thm. 1]{lipman2023flow}
\begin{empheq}[box=\fcolorbox{blue!40!black!10}{green!05}]{equation}
    \label{eq_inversion_formula}
    \utot(x, t) = \mathbb{E}_{z | x, t} \: \ucond(x, z, t)
    \enspace.
\end{empheq}
Hence, the optimal velocity $\utot$ could be approximated by a neural network $u_\theta: \mathbb{R}^d \times [0, 1] \to \mathbb{R}^d$ with parameters $\theta$ by minimizing
\begin{empheq}[box=\fcolorbox{blue!40!black!10}{green!05}]{equation}
\mathcal{L}_{\mathrm{FM}}(\theta) =
    \mathbb{E}_{\substack{ t \sim \mathcal{U}([0, 1]) \\ x_t \sim p(\cdot|t) }}
    \Vert u_\theta(x_t, t) - \utot(x_t, t) \Vert^2
    \enspace.
\end{empheq}
However, $\utot$ is usually (believed) intractable, as a remedy, \citet[Thm. 2]{lipman2023flow} showed that $\mathcal{L}_{\mathrm{FM}}(\theta)$ is equal, up to a constant, to the conditional flow matching loss.
With the choices~\ref{cond_item:1} and~\ref{cond_item:2} made above, the conditional flow matching loss reads

\begin{empheq}[box=\fcolorbox{blue!40!black!10}{green!05}]{equation}
    \label{eq:l_cfm}
    \mathcal{L}_{\mathrm{CFM}}(\theta) =
    \mathbb{E}_{
        \substack{
            x_0 \sim \pnot \\
            x_1 \sim \pdata \\
            t \sim \mathcal{U}([0, 1])\\
             }}
    \Vert u_\theta(x_t, t) - \underbrace{\ucond(x_t, z=x_1, t)}_{
        = \frac{x_1 - x_t}{1-t} = x_1 - x_0
    } \Vert^2
    ,
\end{empheq}
%
where $x_t := (1-t) x_0 + t x_1$.
The objective $\mathcal{L}_{\mathrm{CFM}}$ is easy to approximate, since it is easy to sample from $p_0 = \mathcal{N}(0, \Id)$ and $\mathcal{U} ([0, 1])$;
sampling from $\pdata$ is approximated by sampling from $\pdatahat := \frac{1}{n} \sum_{i=1}^n \delta_{\data{i}}$.
Although it seems natural, replacing $\pdata$ by $\pdatahat$ in \eqref{eq:l_cfm} has a very important consequence: it makes the minimizer $\utothat$ of $\mathcal{L}_{\mathrm{FM}}$ available in closed-form, which we recall below.

\begin{proposition}[Closed-form Formula of the Optimal Velocity]
    \label{prop_closed_form_velocity}
    When $\pdata$ is replaced by $\pdatahat$, with the previous choices~\ref{cond_item:1} and~\ref{cond_item:2},
    the optimal velocity field $\utothat$ in \eqref{eq_inversion_formula} has a closed-form formula:
    \begin{empheq}[box=\fcolorbox{blue!40!black!10}{green!05}]{align}
        \label{eq:closed_form}
        \utothat(x, t) =
            \sum_{i=1}^n \lambda_i(x, t) \frac{\data{i} - x}{1-t} \enspace,
        \end{empheq}
    with      $        \lambda(x, t) = \softmax (
                ( -\frac{\Vert x - t \data{j}\Vert^2}{2(1 -t)^2}
                )_{j=1,\ldots, n}  ) \in \mathbb{R}^n.$
\end{proposition}
The notation $\utothat$ emphasizes the velocity field is optimal for the \emph{empirical} probability distribution $\pdatahat$, not the true one $\pdata$.
Since $\ucond(x, z=\data{i}, t) \propto \data{i} - x$, the optimal velocity field $\utothat$ is a weighted average of the $n$ different directions $\data{i} - x$.
Note that the closed-form formula in~\Cref{eq:closed_form}
 can be found in various previous works, \eg  \citet[Eq. 3]{kamb2024analytic}, \citet{Biroli2024}, \citet{gao2024flow}, \citet{li2024goodscore} or \citet{Scarvelis2023}, and
can be generalized to other choices of continuous distribution $p_0$ (\eg the uniform distribution, see \Cref{app:pf_closed_form_velocity}).

From \Cref{eq:closed_form}, as $t \to 1$, the velocity field $\utothat$ diverges at any point $x$ that does not coincide with one of the training samples $\data{i}$, and it points in the direction of the nearest $\data{i}$.
This creates a paradox: solving the ordinary differential equation \eqref{eq:ode} with the velocity field $\utothat$ can only produce training samples $\data{i}$ (see \citealt[Thm. 4.6]{gao2024flow} for a formal proof).
Therefore, in practice, exactly minimizing the conditional flow matching loss would result in $u_\theta = \utothat$, meaning the model memorizes the training data and fails to generalize.
This naturally yields the following question:

\begin{center}
    \emph{How can flow matching generalize if the optimal velocity field only generates training samples?}
\end{center}

\section{Investigating the key sources of generalization}
\label{sec:investigategeneralization}
In this section, we investigate the key sources of flow matching generalization using the closed-form formula of its velocity field.
First in \Cref{sub:target_sto} we challenge the claim that generalization stems from the stochastic approximation $\ucond$ of the optimal velocity field $\utothat$.
Then, in \Cref{sub:phasesapprox} we show that generalization arises when $\utheta$ fails to approximate the perfect velocity $\utothat$.
Interestingly, the target velocity estimation particularly fails at two critical time intervals.
\Cref{sub:gentimes} shows that one of these critical times is particularly important for generalization.

\subsection{Target stochasticity is not what you need}
\label{sub:target_sto}
\begin{figure}[htbp]
    \centering
    \begin{subcaptionbox}{\textbf{Non stochasticity of $\utothat$ for high-dimensional real data}. \label{fig:hist_cosine}}[0.95\linewidth]
        {
            \centering
            \includegraphics[height=120px]{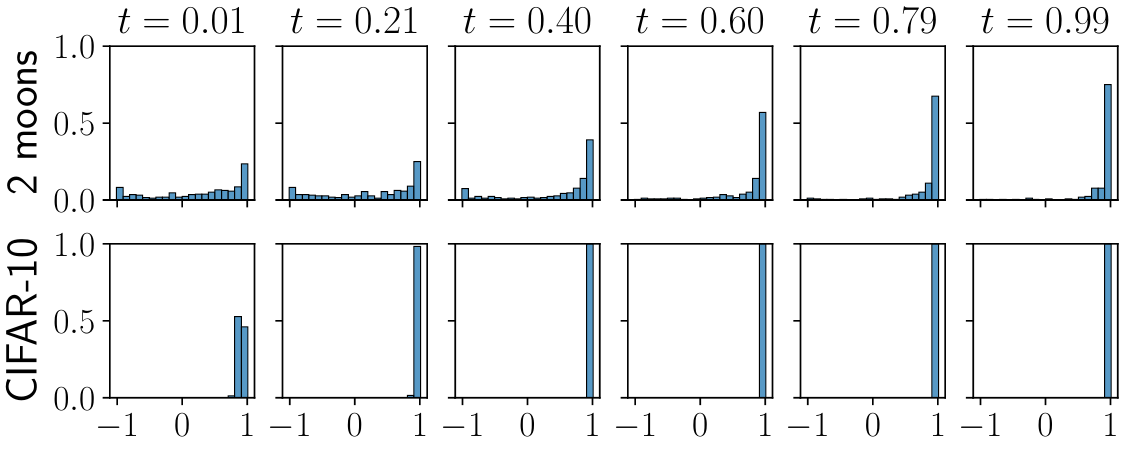}
        }
    \end{subcaptionbox}
    \hfill
    \begin{subcaptionbox}{\textbf{Stochasticity vs. non-stochasticity}\label{fig:illu_nonsto}}[0.45\linewidth]
        {
            \includegraphics[height=100px]{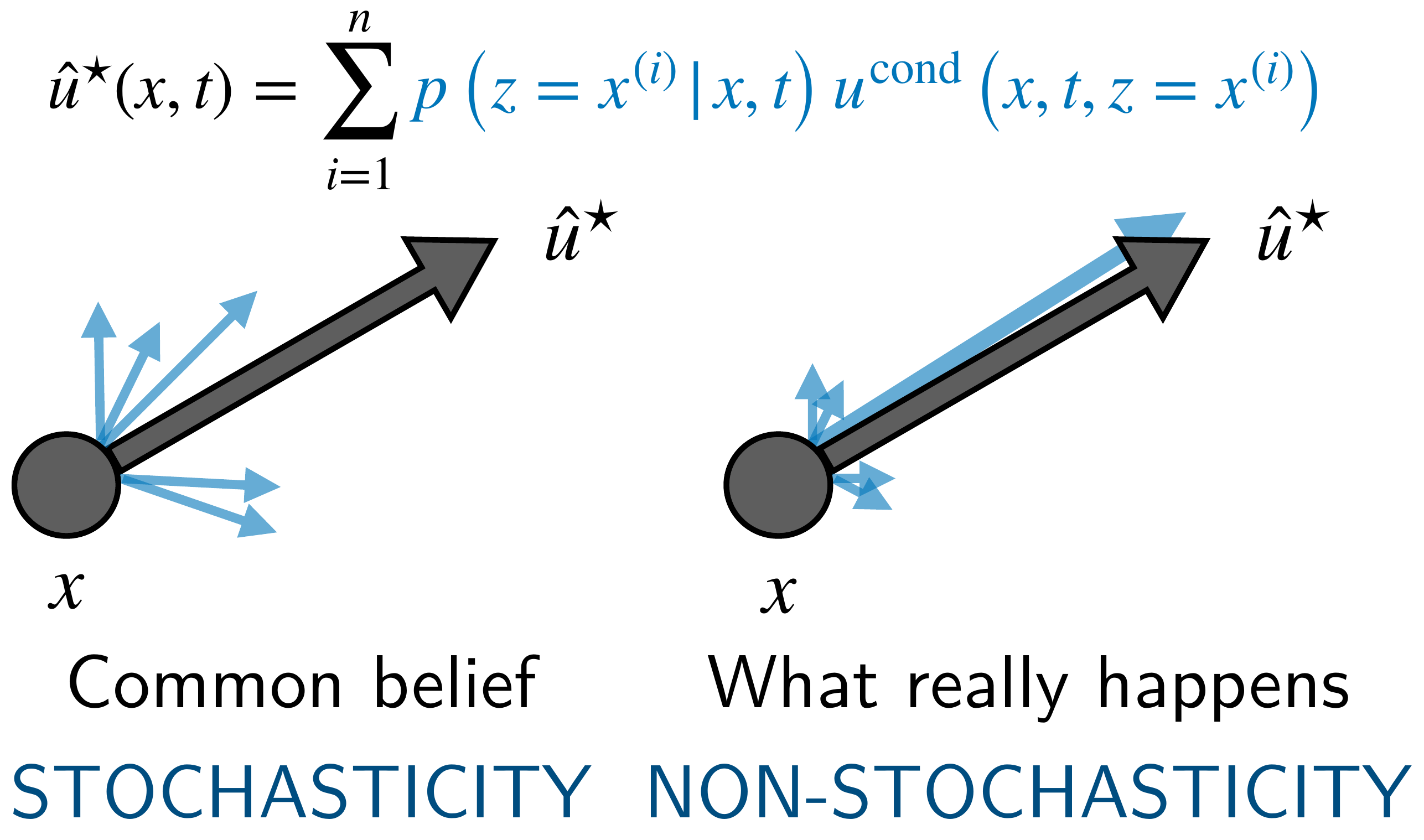}
        }
    \end{subcaptionbox}
    \begin{subcaptionbox}{\textbf{Dimension Dependence}\label{fig:collapse_times}}[0.45\linewidth]
        {
            \includegraphics[height=100px]{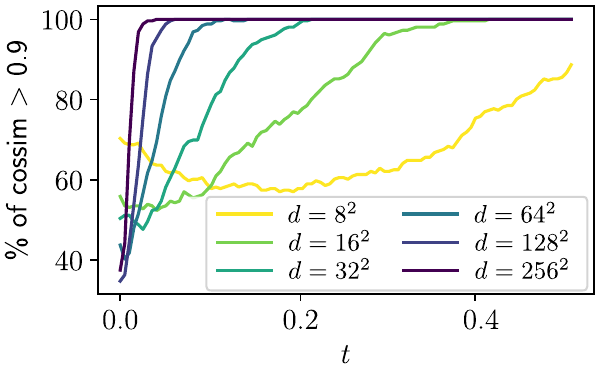}
        }
    \end{subcaptionbox}
    \caption{We challenge the hypothesis that target stochasticity plays a major role in flow matching generalization.
    In \Cref{fig:hist_cosine}, the histograms of the cosine similarities between $\utothat((1-t) x_0 + t x_1, t)$ and $\ucond((1-t)  x_0 + t  x_1, z=x_1, t) = x_1 - x_0$ are displayed for various time values $t$ and two datasets.
    \textit{For real, high-dimensional data, non-stochasticity arises very early} (before $t = 0.2$  for CIFAR-10 with dimension $(3,32,32)$).
    \Cref{fig:collapse_times} displays the alignment between $\utothat$ and $\ucond$ over time for varying image dimensions $d$ on Imagenette.
    }
    \label{fig:main}
\end{figure}

One recent hypothesis is that generalization arises from the fact that the regression target $\ucond$ of conditional flow matching is only a stochastic estimate of $\utothat$.
The fact that the target regression objective only equals the true objective on average is referred to by \citet{Vastola2025generalization} as ``generalization through variance''.
To challenge this assumption, we leverage \Cref{prop_closed_form_velocity}, which
states that the optimal velocity field $\utothat(x, t)$ is a weighted sum of the $n$ values of $\ucond(x, t, z=\data{i} ) = \frac{x^{(i)} - x}{1 - t}$, for $i \in [n]$, and show that, after a \emph{small time value} $t$, this average is in practice equal to a single value in the expectation (see \Cref{fig:illu_nonsto,fig:hist_cosine}).

\xhdr{Comments on \Cref{fig:hist_cosine}}
To produce \Cref{fig:hist_cosine}, we sample $256$ pairs $(x_0, x_1)$ from $p_0 \times \hat{p}_\mathrm{data}$. For each value of $t$, we compute the cosine similarity between the optimal velocity field $\utothat((1 - t) x_0 + t x_1, t)$ and the conditional target $\ucond((1 - t) x_0 + t x_1, z=x_1, t) = x_1 - x_0$. The resulting similarities are aggregated and shown as histograms.
The top row displays the results for the two-moons toy dataset ($d=2$), and the bottom row displays the results for the CIFAR-10 dataset (\citealt{Krizhevsky2009}, $d= 3072$); $n=50$k for both.
%
As $t$ increases, the histograms become increasingly concentrated around $1$, indicating that $\utothat$ aligns closely with a single conditional vector $\ucond$.
From \Cref{eq:closed_form}, this corresponds to a collapse towards 0 of all but one of the softmax weights $\lambda_i(x_t, t)$.
This time corresponds to the collapse time studied by \citet{Biroli2024} for diffusion; we discuss the connection in the related works (\Cref{sec:relatedwork}).
On the two-moons toy dataset, this transition occurs for intermediate-to-large values of $t$, echoing the observations made in low-dimensional settings by \citet[Figure 1]{Vastola2025generalization}.
In contrast, for high-dimensional real datasets, $\utothat(x, t)$ aligns with a single conditional velocity field $x^{(i)} - x$, even at early time steps, suggesting that the non-stochastic regime dominates most of the generative process.
This key difference between low- and high-dimensional data suggests that the transition time between the stochastic and non-stochastic regimes is strongly influenced by the dimensionality of the data.

\xhdr{Comments on \Cref{fig:collapse_times}}
To further illustrate the strong impact of dimensionality, \Cref{fig:collapse_times} reports the proportion of samples $x_t$ (from a batch of 256) for which the cosine similarity between $\utothat$ and  $\ucond \propto x^{(i)} - x$ exceeds $0.9$, as a function of time $t$. This analysis is performed across multiple spatial resolutions of the Imagenette dataset \citep{Howard_Imagenette_2019}, obtaining $\mathrm{dim} \times \mathrm{dim}$ images by spatial subsampling.
\Cref{fig:collapse_times} reveals a sharp transition: as the dimensionality increases, the proportion of high-cosine matches rapidly converges to 100\%.
A practical implication of this behavior is that, for sufficiently large $t$, if $x_0 \sim p_0$ and $\data{i} \sim \hat{p}_\mathrm{data}$,
then $\utothat((1 - t) x_0 + t \data{i}, t)$ is approximately proportional to $\data{i} - x$. Consequently, regressing on $\data{i}$ or on the conditional velocity $x_1 - x_0$ becomes effectively equivalent.
\Cref{sec:learning} investigates how to learn regressing against optimal velocity field $\utothat$, and empirically shows similar results between stochastic and non-stochastic targets.

The regime where flow matching matches stochasticity is mostly concentrated on a very short time interval, for small values of $t$.
We hypothesize that the phenomenon observed here on the optimal velocity field $\utothat$ has major implications on the \emph{learned} flow matching model $u_\theta$, which we further inspect in the next section.

\subsection{Failure to learn the optimal velocity field}
\label{sub:phasesapprox}
This subsection investigates how well the learned velocity field $\utheta$ approximates the optimal/ideal velocity field $\utothat$, and how the quality of this approximation correlates with generalization.
To do so, we propose the following experiment.

\xhdr{Set up of \Cref{fig_dist_ustar}}
To build \Cref{fig_dist_ustar}, we subsampled the CIFAR-10 dataset from $10$ to $10^4$ samples.
For each size, we trained a flow matching model using a standard $34$ million-parameter U-Net (see \Cref{app:expes_details} for details).
Following \citet{kadkhodaie2024generalization}, the number of parameters of the network $\utheta$ remains fixed across dataset sizes.
Importantly, the optimal velocity field $\utothat$ itself depends on the dataset size: as the number of samples increases, the complexity of $\utothat$ also grows.
Thus, we expect the network $\utheta$ to accurately approximate the optimal velocity field $\utothat$ for smaller dataset sizes.

\begin{figure}[tb]
    \begin{center}
        \includegraphics[width=1\linewidth]{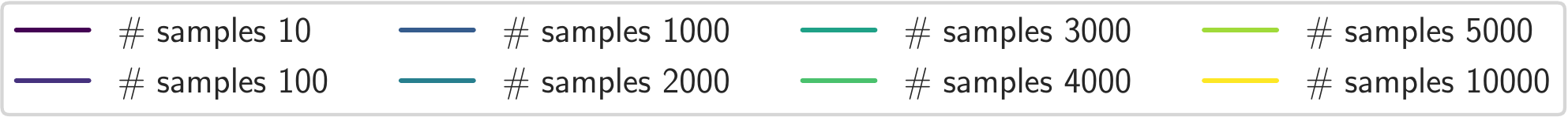}
        \includegraphics[width=1\linewidth]{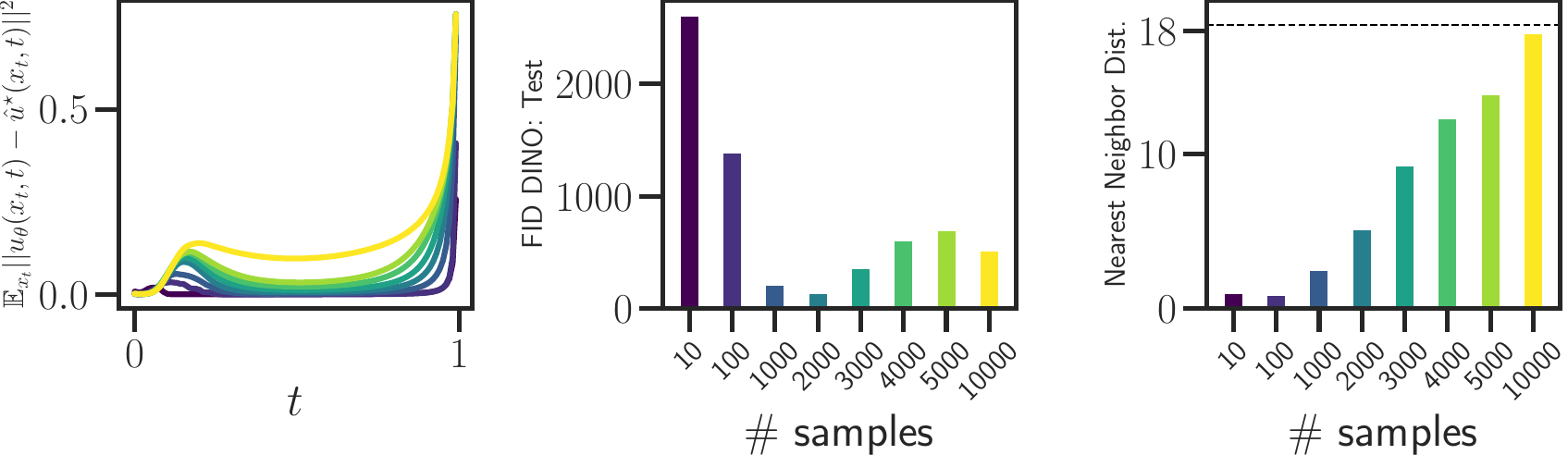}
    \end{center}
    \caption{\textbf{Failure to learn the optimal velocity field, CIFAR-10}. \emph{Left}: The leftmost figure represents the average error between the optimal empirical velocity field $\utothat$ and the learned velocity $\utheta$ for multiple values of time $t$.
    \emph{Middle}:
    The middle figure displays the FID-10k computed on the test dataset, using the DINOv2 embedding.
    \emph{Right}: The rightmost figure displays the average distance between the generated samples and their closest image from the training set -- for reference, the horizontal dashed line indicates the mean distance between an image of CIFAR-10 train and its nearest neighbor in the dataset.
    All the quantities are computed/learned on a varying number of training samples ($10$ to $10^4$) of the CIFAR-10 dataset.
    }
    \label{fig_dist_ustar}
\end{figure}

\xhdr{Comments on \Cref{fig_dist_ustar}}
The leftmost plot shows the average training error
\begin{equation*}
    \mathbb{E}_{\substack{x_0 \sim p_0 \\ x_1 \sim \pdatahat}}
    \left\| u_\theta(x_t, t) - \utothat(x_t, t) \right\|^2, \quad \text{where} \quad x_t := (1-t)x_0 + t x_1
    \enspace,
\end{equation*}
between the learned velocity $\utheta$ and the optimal empirical velocity field $\utothat$, evaluated across multiple time values $t$ and dataset sizes.
With only $10$ samples (darkest curve), the network $\utheta$ closely approximates $\utothat$. As the dataset size increases, the complexity of $\utothat$ grows, and the approximation by $\utheta$ becomes less accurate.
In particular, the approximation fails at two specific time intervals: around $t \approx 0.15$ and near $t = 1$.
The failure near $t = 1$ is expected, as $\utothat$ becomes non-Lipschitz at $t=1$.
Interestingly, the early-time failure at $t \approx 0.15$ corresponds to the regime where $\utothat$ and $\ucond$ start to correlate (see \Cref{fig:hist_cosine} in \Cref{sub:target_sto}).
The middle plot of \Cref{fig_dist_ustar} reports the FID-10k, computed on the test set in the DINOv2 embedding space \citep{dinov2}, for various dataset sizes.
For a small dataset (\eg $ \# \text{samples} = 10$), $\utheta$ approximates $\utothat$ well but does not generalize -- the test FID exceeds $10^3$.
As the dataset size increases ($1000 \leq \# \text{samples} \leq 3000$), the approximation $\utheta$ becomes less accurate. Despite this, the model achieves lower FID scores on the test set but still memorizes the training data.
The rightmost plot of \Cref{fig_dist_ustar} illustrates this memorization by showing the average distance between each generated sample and its nearest neighbor in the training set.
For larger datasets ($\# \text{samples} \geq 3000$), this distance increases substantially, indicating that the model generalizes better.
Overall, \Cref{fig_dist_ustar} also suggests that the FID metric can be misleading, even when computed on the test set. For example, the model trained with 1000 samples has a low test FID but memorizes training examples.

\Cref{fig_dist_ustar} confirms that generalization arises when the network $\utheta$ fails to estimate the optimal velocity field $\utothat$, and that this failure occurs at two specific time intervals.
In \Cref{sub:gentimes}, we investigate which of these two intervals is responsible for driving generalization.

\subsection{When does generalization arise?}
\label{sub:gentimes}
To investigate whether the failure to approximate $\utothat$ matters the most at small or large values of $t$, we carry out the following experiment.

\begin{figure}[t]
    \begin{center}
        \includegraphics[width=0.48\linewidth]{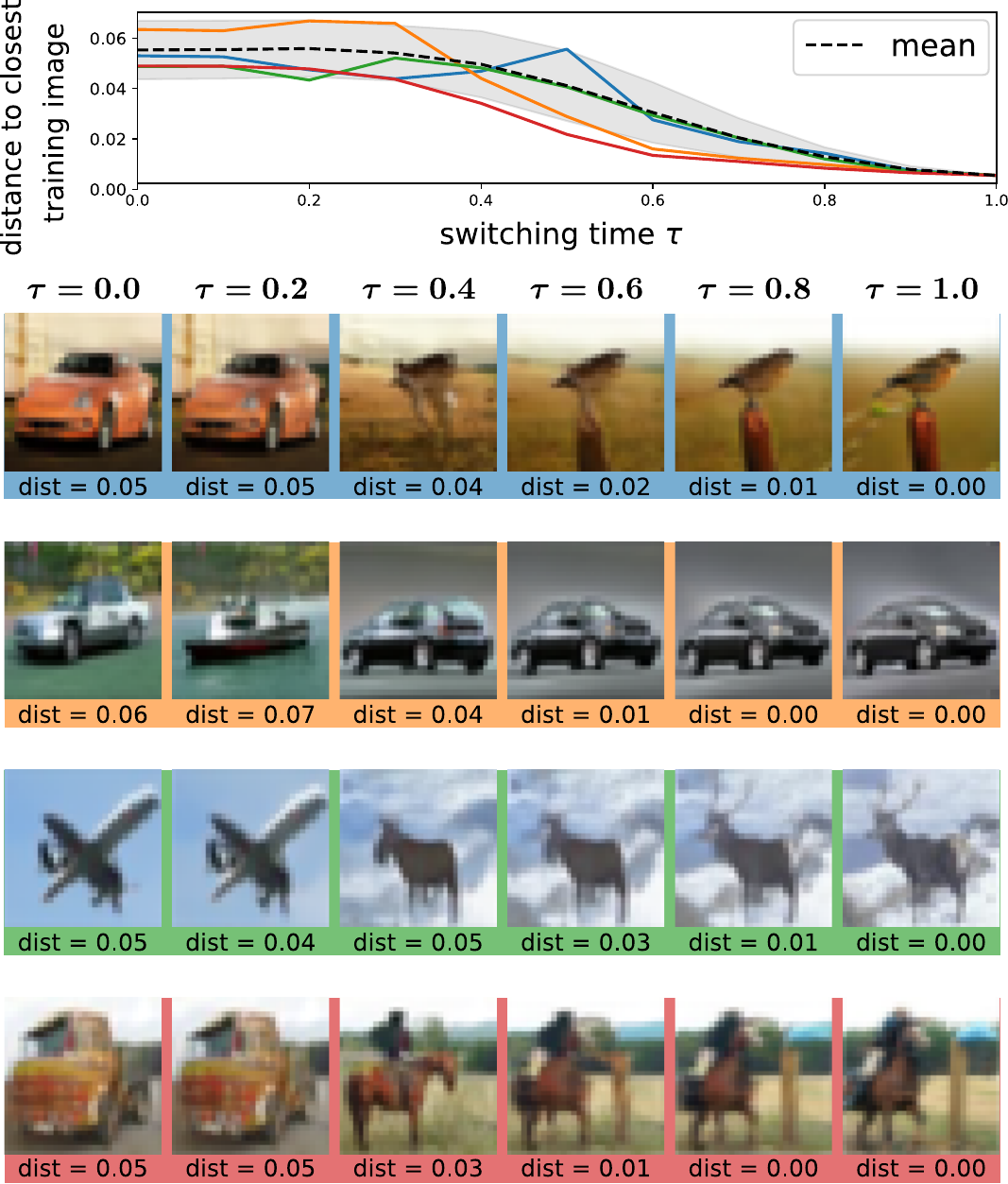}
        \includegraphics[width=0.48\linewidth]{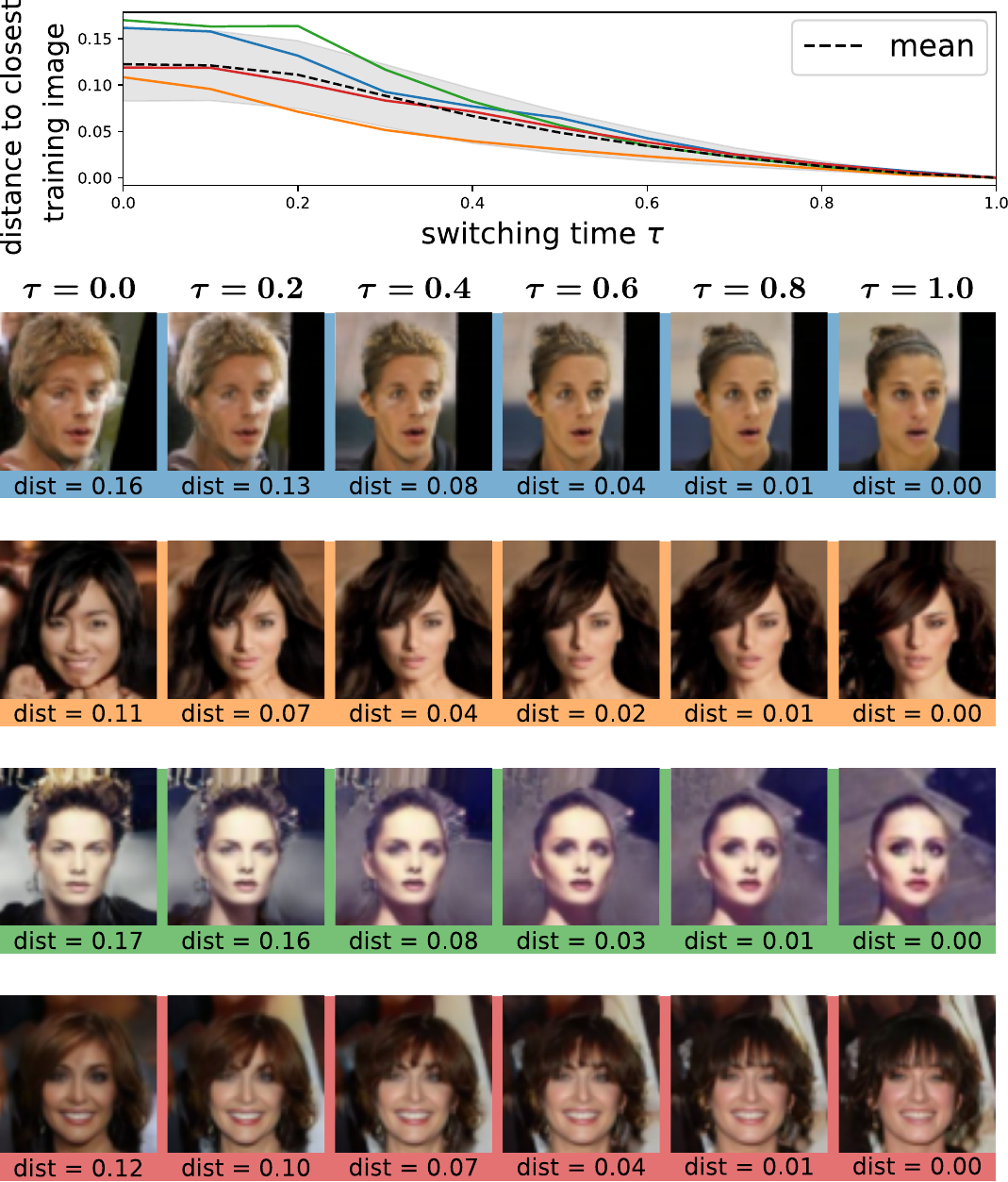}
    \end{center}
    \caption{\textbf{Generalization occurs at small times on CIFAR-10 (left) and CelebA $64$ (right)}. \emph{Top}: Generalization (distance between generated samples and training data) of hybrid models that follow $\utothat$ on $[0, \tau]$, then $u_\theta$ on $[\tau, 1]$. The four colored curves correspond to four specific $x_0$, the black dashed curve is the mean distance over the 256 generated images.
    \emph{Bottom}: visualization of generated images for the four different starting noises and various values of $\tau$ (the background color matching the curve in the top figure).
    \emph{Following $\utothat$ until $\tau \geq 0.3$ yields a model that is not able to generalize}.}
    \label{fig:umix}
\end{figure}

\xhdr{Set up of \Cref{fig:umix}}
\looseness-1
We first learn a velocity field $u_\theta$ using standard conditional flow matching (see \Cref{app:expes_details}), then we construct a hybrid model: we define a piecewise trajectory where the flow is governed by the optimal velocity field $\utothat$ for times $t \in [0, \tau]$, and by the learned velocity field $u_\theta$ for times $t \in [\tau, 1]$, for a given threshold parameter $\tau \in [0, 1]$.
For the extreme case $\tau=1$, the full trajectory follows $\utothat$, and samples exactly match training data points.
Conversely, when $\tau = 0$, the entire trajectory is governed by $u_\theta$, yielding novel samples.
Intermediate values of $\tau$ produce a mixture of both behaviors, which we interpret as reflecting varying degrees of generalization.
%
To assess generalization, we measure the distance of generated samples to the dataset using the LPIPS metric \citep{zhang2018unreasonable}, which computes the feature distance between two images via some pretrained classification network.
We define the distance of a generated sample $x$ to a dataset $\mathcal D = \{\data{1}, \ldots, \data{n} \}$ as $\mathrm{dist}(x, \mathcal D) = \min_{x^{(i)} \in \mathcal{D}} \mathrm{LPIPS}(x, x^{(i)})$.
We fix a random batch of 256 pure noise images from $p_0$.
Then, for various threshold values $\tau$, we generate 256 images with the hybrid model, always starting from this batch.
Finally, we measure the creativity of the hybrid model as the mean of the aforementioned LPIPS distances between the 256 generated samples and the dataset.

\xhdr{Comments on \Cref{fig:umix}}
The top row displays the LPIPS distances as $\tau$ varies, on the CIFAR-10 (left) and CelebA - $64 \times 64$ (right) datasets.
For $\tau \leq 0.2$, the hybrid model remains as creative as $u_\theta$, despite following $\utothat$ in the first steps.
For $\tau > 0.2$, the LPIPS distance starts dropping.
On the displayed generated samples (bottom rows), we in fact see that as soon as $\tau \geq 0.4$, the sample generated by the hybrid model is almost the same as the one obtained with $\utothat$ $(\tau=1)$.
This means
\emph{the final image is already determined at $t=0.4$}, and despite the generalization capacity of the learned velocity field $\utheta$, following it only after $t \geq 0.4$ is not enough to create a new image: \emph{generalization occurs early and cannot fully be explained by the failure to correctly approximate $\utot$ at large $t$.}

Although we have shown that the stochastic phase was limited to small values of $t$ in real-data settings, we have not yet definitively ruled it out as the cause of generalization. In the following \Cref{sec:learning}, we introduce a learning procedure designed to address this question directly.


\section{Learning with the closed-form formula}
\label{sec:learning}

In this section, in order to discard the impact of stochastic target on the generalization,  we propose to directly regress against the closed-form formula in~\Cref{eq:closed_form}.

\subsection{Empirical flow matching}

Regressing against the closed-form $\utothat$, defined in \Cref{eq:closed_form}, at a point $(x_t, t)$  requires computing a weighted sum of the conditional velocity fields over \emph{all} the $n$ training points $\data{i}$.
For a dataset of $n$ samples of size $d$, and a batch of size $|\mathcal{B}|$, computing the weights of the exact closed-form formula $\utothat(x, t)$ of flow matching  requires
$\mathcal{O}(n \times |\mathcal{B}| \times d) $.
These computations are prohibitive since they must be performed for each batch.
One natural idea is to estimate the closed-form formula $\utothat$ (\Cref{eq:closed_form}), by a Monte Carlo approximation (\Cref{eq:l_efm_softmax}), using $M \leq n$ samples $\batch{1}, \dots, \batch{M}$:
\begin{empheq}[box=\fcolorbox{blue!40!black!10}{green!05}]{align}\label{eq:l_efm}
    \mathcal{L}_{\mathrm{EFM}}(\theta)
    & =
    \mathbb{E}_{
        \substack{
        x_0 \sim p_0 \\
        x_1 \sim \pdatahat \\
        t \sim \mathcal{U}([0, 1]) \\
        \batch{2}, \dots, \batch{M} \sim \pdatahat
        }}
    \Vert u_\theta(x_t, t) -
    \utothatemp(x_t, t) \Vert^2
    \enspace,
\end{empheq}
with
$x_t = (1-t) x_0 + t x_1$,
$\batch{1} := x_1$, and
\begin{align}
    \utothatemp(x, t)
    & =
    \sum_{j=1}^{M}
    \lambda(x, t)
    \frac{\batch{j} - x}{1-t}
    \enspace
    ,
    \enspace
    \lambda(x, t) = \softmax \left (
        \left ( -\frac{\Vert x - t \data{l}\Vert^2}{2(1 -t)^2}
        \right )_{l=1,\ldots, n} \right )
    \enspace.
    \label{eq:l_efm_softmax}
\end{align}

The formulation in \Cref{eq:l_efm} may appear naive at first glance. Still, it hinges on a crucial trick: the Monte Carlo estimate is computed using a batch that systematically includes the point $x_1$, that generated the current $x_t$.
If instead $\batch{1}$ were sampled independently from $\pdatahat$, this could introduce a sampling bias (see \citealt{Ryzhakov2024explicit}, \Cref{app_efm}, and the corresponding OpenReview comments\footnote{\url{https://openreview.net/forum?id=XYDMAckWMa}} for an in-depth discussion).
\Cref{app_prop_efm} shows that the estimate $\utothatemp$ is unbiased and has lower variance than the standard conditional flow matching target.
\begin{restatable}{proposition}{propefm}
\label{app_prop_efm}
    We denote the conditional probability distribution $p(z = \data{i} \mid x, t)$ over $\{\data{i}\}_{i=1}^n$ by $\pdatahat(z \mid x, t)$.
    With no constraints on the learned velocity field $\utheta$,
    \begin{propenum}[label=\roman*)]
        \item \label{app_prop_minimizer} The minimizer of \Cref{eq:l_efm} writes, for all $(x, t)$
        \begin{align}
            \mathbb{E}_{
                \substack{
                    \batch{1} \sim \pdatahat(\cdot | x, t ) \\
                    \batch{2}, \dots, \batch{M} \sim \pdatahat }
                }
        \left[ \utothatemp(x, t) \right] \enspace.
        \end{align}
    \item \label{app_prop_unbiased}
        In addition, for all $(x, t)$,
        the minimizer of \Cref{eq:l_efm} equals the optimal velocity field, i.e.,
    \begin{align}
        \mathbb{E}_{
            \substack{
                \batch{1} \sim \pdatahat(\cdot | x, t ) \\
                \batch{2}, \dots, \batch{M} \sim \pdatahat }
            }
        \left[ \utothatemp(x, t) \right]
        = \utothat(x, t)
        \enspace.
    \end{align}
    \item \label{app_prop_variance} The conditional variance of the estimator $\utothatemp$ is smaller than the usual conditional variance:
    \begin{align}
        \Var_{
            \substack{
                \batch{1} \sim \pdatahat(\cdot | x, t ) \\
                \batch{2}, \dots, \batch{M} \sim \pdatahat }
            }
            \left[ \utothatemp(x, t) \right]
        \leq
        \Var_{\batch{1} \sim \pdatahat(\cdot | x, t )} \left[ \ucond(x, \batch{1}, t) \right].
    \end{align}
    \end{propenum}
\end{restatable}
The proof of \Cref{app_prop_efm} is provided in \Cref{app_sub_proof_efm}.
The estimator $\utothatemp$ of the optimal field $\utothat$ is closely related to self-normalized importance sampling (see \Cref{app_sub_theoretical_res} and \citealt[Chap. 9.2]{Owen2013montecarlobook}), as well as to Rao-Blackwellized estimators \citep{casella1996raoblaclwell,cardoso2022biasreduced}.
As discussed in \citet{Ryzhakov2024explicit}, self-normalized importance sampling estimators of $\utothat$ are generally biased, in the sense that:
\begin{math}
    \mathbb{E}_{\batch{1}, \dots, \batch{M} \sim \pdatahat} \utothatemp(x_t, t) \neq \utothat(x_t, t)
    \enspace.
\end{math}
A key insight is that our estimator includes
$\batch{1} \sim \pdatahat(\cdot \mid x_t, t)$,
which leads to the main result of \Cref{app_prop_efm}.
In \Cref{sub:expeslargescale}, we demonstrate that \Cref{alg:efm}, designed to solve \Cref{eq:l_efm}, yields consistent improvements on high-dimensional datasets such as CIFAR-10 and CelebA.
Additional details on the unbiasedness of $\mathcal{L}_{\mathrm{EFM}}$ can be found in the supplementary material (\Cref{app_efm}).
From a computational perspective, despite requiring $M$ additional samples, \Cref{alg:efm} remains significantly more efficient than increasing the batch size by a factor of $M$: the $M$ samples are merely averaged (with weights), while the backpropagation remains identical to that of \Cref{alg:cfm}.

\begin{minipage}[t]{0.41\textwidth}
    \begin{algorithm}[H]
    \caption{Vanilla Flow Matching}
    \label{alg:cfm}
    \SetKwInOut{Input}{input}
    \SetKwInOut{Init}{init}
    \SetKwInOut{Parameter}{param}
    \For{$k$ in $1, \dots, n_{\mathrm{iter}}$ }
    {
    $t \sim \mathcal{U}([0, 1]) $

    $x_0 \sim \mathcal{N}(0, \Id)$,
    $x_1 \sim \pdatahat$,

    $x_t = (1 - t) x_0 + t x_1$

    $\textcolor{blue}{\ucond(x_t, t) = \dfrac{x_1 - x_t}{1 - t} = x_1 - x_0}$

    $\mathcal{L}(\theta) = \norm{\utheta(x_t, t) - \textcolor{blue}{\ucond(x_t, t)}}^2$

    Compute $\nabla \mathcal{L}(\theta)$ and update $\theta$

    }
\Return{$\utheta$}
\end{algorithm}
\end{minipage}
\hfill
\begin{minipage}[t]{0.60\textwidth}
    \begin{algorithm}[H]
    \caption{Empirical Flow Matching}
    \label{alg:efm}
    \SetKwInOut{Input}{input}
    \SetKwInOut{Init}{init}
    \SetKwInOut{Parameter}{param}
    \Parameter{$M$ \textcolor{blue}{// Number of samples in the empirical mean}}
    \For{$k$ in $1, \dots, n_{\mathrm{iter}}$ }
    {

    $x_0 \sim \mathcal{N}(0, \Id)$,
    $x_1 \sim \pdatahat$,
    $t \sim \mathcal{U}([0, 1]) $

    $x_t = (1 - t) x_0 + t x_1$

    $b^{(1)} = x_1$

    $\forall j\!\in\! \llbracket 2{,}M\rrbracket$, $b^{(j)} \sim \pdatahat$
    \tcp*[l]{Samples from $\pdatahat$}

    $\textcolor{blue}{\utothat_M(x_t, t)
    =
    \sum_{j=1}^M
    \frac{\batch{j} - x_t}{1-t}
    \cdot
    \left[
    \softmax \left (-\frac{\normin{x_t - t \cdot b }^2}{2(1 -t)^2} \right )
    \right]_j}$

    $\mathcal{L}(\theta) = \norm{\utheta(x_t, t) - \textcolor{blue}{\utothat_M(x_t, t)} }^2$

    Compute $\nabla \mathcal{L}(\theta)$ and update $\theta$

    }
\Return{$\utheta$}
\end{algorithm}

\end{minipage}

\subsection{Experiments}
\label{sub:expeslargescale}

We now learn with empirical flow matching (EFM, \Cref{eq:l_efm} and \Cref{alg:efm}) in practical high-dimensional settings.
Our goal with this empirical investigation is first to observe if regressing against a more deterministic target leads to performance improvement/degradation.

\xhdr{Datasets and Models}
We perform experiments on the image datasets CIFAR-10 \citep{Krizhevsky2009} and CelebA $64 \times 64$ \citep{celeba}.
For the experiments, we compare vanilla conditional flow matching \citep{lipman2023flow,liu2023rectifiedflow,albergo2023stochasticinterpolant}, optimal transport flow matching \citep{pooladian2023multisample,tong2023improving}, and the empirical flow matching in \Cref{alg:efm}, for multiple numbers of samples $M$ to estimate the empirical mean.
Training details are in \Cref{app:expes_details}.

\xhdr{Metrics}
To assess generalization performance, we use the standard Fréchet Inception Distance \citep{heusel2017gans} with Inception-V3 \citep{szegedy2016rethinking} but we also follow the recommendation of \citet{stein2023evalgenmodels} using the DINOv2 embedding \citep{oquab2023dinov2}, which is known to a more expressive and discriminative embedding, that leads to a less biased evaluation.
We also measure the FID between the generated and the train and test sets, rather than only on the training set, as is often done in generative modeling benchmarks.
The train FID is computed between $50$k generated images and the $50$k images from the training set. The test FID is computed between the same $50$k generated images and the $10$k images from the test set.
On \Cref{fig_dist_ustar}, we also displayed a memorization metric that would detect a pure copy of the training set.
Overall, defining and quantifying the generalization ability of generative models is overall a challenging task: train and test FID are known to be imperfect \citep{stein2023evalgenmodels,Jiralerspong2023fld,parmar2021cleanfid}, yet no superior competitor has emerged.
%

\begin{figure}[tb]
    \begin{center}
        \includegraphics[width=1\linewidth]{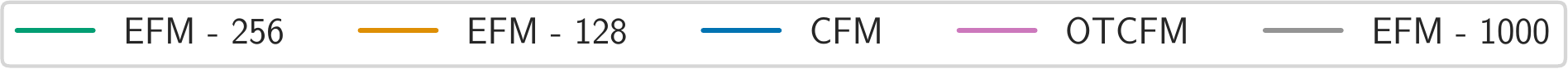}
        \includegraphics[width=1\linewidth]{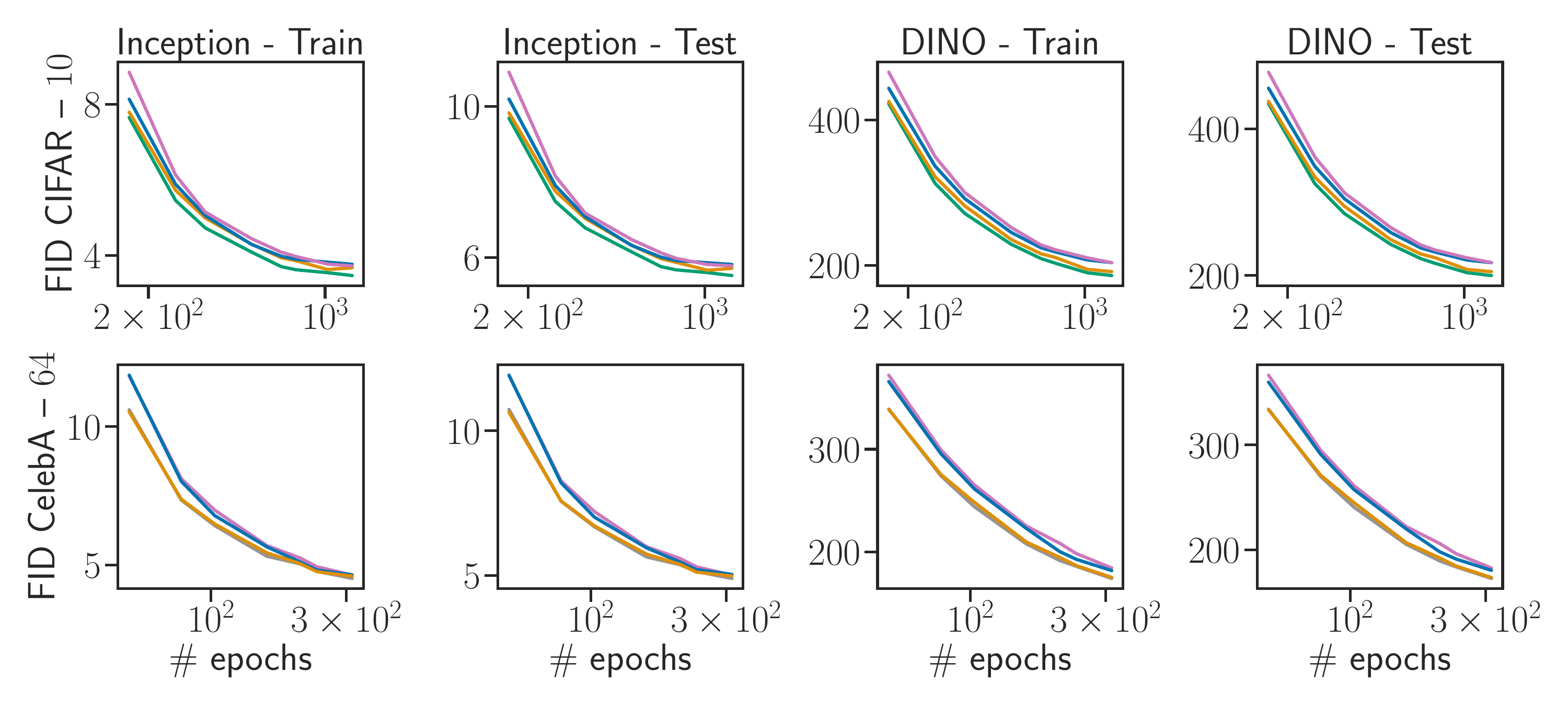}
    \end{center}
    \caption{FID computed on the training set ($50$k) and the test set ($10$k) using multiple embeddings, Inception and DINOv2. Regressing against a more deterministic target (EFM - $128$, $256$, $1000$) does not yield performance decreases. On the contrary, the more deterministic the target, the better the performance.}
    \label{fig_fid_cifar}
\end{figure}

\xhdr{Comments on \Cref{fig_fid_cifar}}
\Cref{fig_fid_cifar} compares vanilla flow matching, OTCFM, and the empirical flow matching (EFM, \Cref{alg:efm}) approaches using various numbers of samples to estimate the empirical mean, $M \in \{128, 256, 1000\}$.
First, we observe that learning with a more deterministic target does not degrade either training or testing performance, across both types of embeddings.
On the contrary, we consistently observe modest but steady improvements as stochasticity is reduced. For both CIFAR-10 and CelebA, increasing the number of samples $M$ used to compute the empirical mean—\ie, making the targets less stochastic—leads to more stable improvements.
It is worth noting that \Cref{alg:efm} has a computational complexity of $\mathcal{O}(M \times |\mathcal{B}| \times d)$, where $|\mathcal{B}|$ is the batch size, $M$ is the number of samples used to estimate the empirical mean, and $d$ is the sample dimension. In our experiments, choosing $M= |\mathcal{B}| =128$ yielded a modest time overhead.
For empirical flow matching, we experimented with several values beyond $M = 1000$ (\eg $M = 2000$, $M = 5000$). The results were nearly identical to those obtained with $M = 1000$, with curves being visually indistinguishable. Therefore, we chose not to report results for $M \geq 1000$.



\section{Related work}
\label{sec:relatedwork}

The existing literature related to our study can be roughly divided into three approaches: leveraging the closed-form, studies on the memorization vs generalization, and characterization of the different phases of the generating dynamics.

\xhdr{Leveraging the closed-form} \Cref{prop_closed_form_velocity} has been leveraged in several ways.
The closest existing work is by \citet{Ryzhakov2024explicit}, who propose to regress against $\utothat$ as we do in \Cref{sec:learning}.
Nevertheless, their motivation is that reducing the variance of the velocity field estimation makes learning more accurate: as explained in \Cref{sub:target_sto}, we argue this claim rests on misleading $2$D-based intuitions (\eg Figure 1, challenged by \Cref{sub:target_sto}).
The idea of regressing against a more deterministic target (as \Cref{app_prop_efm} shows) derived from the optimal closed-form velocity field has also been empirically explored for diffusion models \citep{xu2023stabletargetfield}.
\citet{Scarvelis2023} bypass training, and suggest using a smoothed version of $\utothat$ to generate novel samples.
In a work specific to images and convolutional neural networks, \citet{kamb2024analytic} suggested that flow matching indeed ends up learning an optimal velocity, but that instead of memorizing training samples, the velocity memorizes a combination of all possible patches in an image and across the images.
They show remarkable agreement between their theory and the trajectories followed by learned vector fields, but their work is limited to convolutional architectures, and was recently extended to a larger class of architectures \citep{lukoianov2025locality}.

\xhdr{Memorization and reasons for generalization} \citet{kadkhodaie2024generalization} directly relates the transition from memorization to generalization to the size of the training dataset, and proposes a geometric interpretation. We provide a complementary experiment in \Cref{sub:phasesapprox}, quantifying how much the network fails to estimate the optimal velocity field.
\citet{Gu2025memorizationdiffusion} provide a detailed experimental investigation into the potential causes for generalization, primarily based on the characteristics of the dataset and choices for training and model.
\citet{Vastola2025generalization} explores different factors of generalization in the case of diffusion, with a special focus on the stochasticity of the target objective in the learning problem.
Through a physic-based modeling of the generative dynamics, they study the covariance matrix of the noisy estimation of the exact score.
In our work, we believe that we have shown that this claim was not valid for real high-dimensional data.
\citet{niedoba2024towards} study the poor approximation of the exact score by the learned models: like \citet{kamb2024analytic}, they suggest that the generalization of the learned models comes from memorization of many patches in the training data.

\xhdr{Temporal regimes} \citet{Biroli2024,sclocchi2025phase} provide an analysis of the exact score, the counterpart of the exact velocity field for diffusion.
For a multimodal target distribution, the authors identify three phases (we keep the convention that $t=0$ is noise and $t=1$ is target): for $t < t_1$, all trajectories are indistinguishable; for $t_1 < t < t_2$, trajectories converging to different modes separate; for $t > t_2$, trajectories all point to the training dataset.
In the case of Gaussians mixtures target, they highlight the dependency of $t_2$ in the dimension and the number of samples, in $\mathcal{O} \left( ({\log n}) /{d} \right)$, meaning that the first phases are observable only if the number of training points is exponential in the dimension. The methodology they adopt to validate the existence of such $t_2$ on real data relies on the stochasticity of the backward generative process, which does not hold in the case of flow matching.
Our experiments on \emph{learned} flow matching models allow us to take this theoretical study on memorization and temporal behaviors of generative processes a step further.
\section{Conclusion, limitations and broader impact}
\looseness-1
\xhdr{Conclusion}
By challenging the assumption that stochasticity in the loss function is a key driver of generalization, our findings help clarify the role of approximation of the exact velocity field in flow matching models.
Beyond the different temporal phases in the generation process that we have identified, we expect further results to be obtained by uncovering new properties of the true velocity field.

\looseness-1
\xhdr{Limitation} Our work is mainly empirical, with a focus on \emph{learned} models, but did not precisely characterize the learned velocity field, in particular, how it behaves outside the trajectories defined by the optimal velocity. Leveraging existing work on the inductive biases of the architectures at hand seems like a promising venue.
Another limitation is that we did not investigate the interaction between the architectural inductive bias, and optimization procedures: this is a very challenging, but active area of research \citep{boursier2025simplicity,bonnaire2025diffusion,favero2025bigger}.

\xhdr{Broader impact} We hope that identifying the key factors of generalization will lead to improved training efficiency.
However, generative models also raise concerns related to misinformation (notably deepfakes), data privacy, and potential misuse in generating synthetic but realistic content.

\newpage
\section{Acknowledgments}
The authors thank the Blaise Pascal Center for its computational support, using the SIDUS \citep{cbp} solution.

\bibliographystyle{abbrvnat}
\bibliography{references.bib}

@article{favero2025bigger,
  title={Bigger Isn't Always Memorizing: Early Stopping Overparameterized Diffusion Models},
  author={A. Favero and A. Sclocchi and M. Wyart},
  journal={NeurIPS},
  year={2025}
}

@article{lukoianov2025locality,
  title={Locality in image diffusion models emerges from data statistics},
  author={A. Lukoianov and C. Yuan and J. Solomon and V. Sitzmann},
  journal={NeurIPS},
  year={2025}
}

@article{bonnaire2025diffusion,
  title={Why Diffusion Models Don't Memorize: The Role of Implicit Dynamical Regularization in Training},
  author={Bonnaire, Tony and Urfin, Rapha{\"e}l and Biroli, Giulio and M{\'e}zard, Marc},
  journal={NeurIPS},
  year={2025}
}

@article{boursier2025simplicity,
  title={Simplicity bias and optimization threshold in two-layer {ReLu} networks},
  author={E. Boursier and N. Flammarion},
  journal={ICML},
  year={2025}
}

@article{sclocchi2025phase,
  title={A phase transition in diffusion models reveals the hierarchical nature of data},
  author={A. Sclocchi and A. Favero and M. Wyart},
  journal={PNAS},
  volume={122},
  number={1},
  pages={e2408799121},
  year={2025},
  publisher={National Academy of Sciences}
}

@software{Howard_Imagenette_2019,
    title={Imagenette: A smaller subset of 10 easily classified classes from Imagenet},
    author={Jeremy Howard},
    year={2019},
    month={March},
    publisher = {GitHub},
    url = {https://github.com/fastai/imagenette}
}

@article{cbp,
  title={SIDUS—the solution for extreme deduplication of an operating system},
  author={Quemener, Emmanuel and Corvellec, Marianne},
  journal={Linux Journal},
  volume={2013},
  number={235},
  pages={3},
  year={2013},
  publisher={Belltown Media Houston, TX}
}

@article{cardoso2022biasreduced,
  title={Br-snis: bias reduced self-normalized importance sampling},
  author={Cardoso, Gabriel and Samsonov, Sergey and Thin, Achille and Moulines, Eric and Olsson, Jimmy},
  journal={NeurIPS},
  volume={35},
  pages={716--729},
  year={2022}
}

@article{casella1996raoblaclwell,
  title={Rao-Blackwellisation of sampling schemes},
  author={Casella, George and Robert, Christian P},
  journal={Biometrika},
  volume={83},
  number={1},
  pages={81--94},
  year={1996},
  publisher={Oxford University Press}
}

@book{Owen2013montecarlobook,
   author = {Art B. Owen},
   year = 2013,
   title = {Monte Carlo theory, methods and examples},
   publisher = {\url{https://artowen.su.domains/mc/}}
}

@book{robert1999montecarlobook,
  title={Monte Carlo statistical methods},
  author={Robert, Christian P and Casella, George and Casella, George},
  volume={2},
  year={1999},
  publisher={Springer}
}

@article{xu2023stabletargetfield,
  title={Stable target field for reduced variance score estimation in diffusion models},
  author={Xu, Yilun and Tong, Shangyuan and Jaakkola, Tommi},
  journal={ICLR},
  year={2023}
}

@article{dinov2,
  title={Dinov2: Learning robust visual features without supervision},
  author={Oquab, Maxime and Darcet, Timothée and Moutakanni, Theo and Vo, Huy V. and Szafraniec, Marc and Khalidov, Vasil and Fernandez, Pierre and Haziza, Daniel and Massa, Francisco and El-Nouby, Alaaeldin and Howes, Russell and Huang, Po-Yao and Xu, Hu and Sharma, Vasu and Li, Shang-Wen and Galuba, Wojciech and Rabbat, Mike and Assran, Mido and Ballas, Nicolas and Synnaeve, Gabriel and Misra, Ishan and Jegou, Herve and Mairal, Julien and Labatut, Patrick and Joulin, Armand and Bojanowski, Piotr},
  journal={TMLR},
  year={2024}
}

@inproceedings{nichol2021improved,
  title={Improved denoising diffusion probabilistic models},
  author={Nichol, Alexander Quinn and Dhariwal, Prafulla},
  booktitle={ICML},
  pages={8162--8171},
  year={2021},
  organization={PMLR}
}

@inproceedings{celeba,
  title = {Deep Learning Face Attributes in the Wild},
  author = {Liu, Ziwei and Luo, Ping and Wang, Xiaogang and Tang, Xiaoou},
  booktitle = {Proceedings of International Conference on Computer Vision (ICCV)},
  month = {December},
  year = {2015}
}

@inproceedings{zhang2024emergencerepro,
  title={The emergence of reproducibility and consistency in diffusion models},
  author={Zhang, Huijie and Zhou, Jinfan and Lu, Yifu and Guo, Minzhe and Wang, Peng and Shen, Liyue and Qu, Qing},
  booktitle={ICML},
  year={2024}
}

@inproceedings{Dar2023investigating,
  title={Investigating data memorization in 3d latent diffusion models for medical image synthesis},
  author={Dar, Salman Ul Hassan and Ghanaat, Arman and Kahmann, Jannik and Ayx, Isabelle and Papavassiliu, Theano and Schoenberg, Stefan O and Engelhardt, Sandy},
  booktitle={International Conference on Medical Image Computing and Computer-Assisted Intervention},
  pages={56--65},
  year={2023},
  organization={Springer}
}

@article{Somepalli2023understanding,
  title={Understanding and mitigating copying in diffusion models},
  author={Somepalli, Gowthami and Singla, Vasu and Goldblum, Micah and Geiping, Jonas and Goldstein, Tom},
  journal={NeurIPS},
  volume={36},
  pages={47783--47803},
  year={2023}
}

@article{li2024goodscore,
  title={A good score does not lead to a good generative model},
  author={Li, Sixu and Chen, Shi and Li, Qin},
  journal={arXiv preprint arXiv:2401.04856},
  year={2024}
}

@inproceedings{yoon2023diffusion,
  title={Diffusion probabilistic models generalize when they fail to memorize},
  author={Yoon, TaeHo and Choi, Joo Young and Kwon, Sehyun and Ryu, Ernest K},
  booktitle={ICML 2023 workshop on structured probabilistic inference \& generative modeling},
  year={2023}
}

@article{niedoba2024towards,
  title={Towards a Mechanistic Explanation of Diffusion Model Generalization},
  author={Niedoba, Matthew and Zwartsenberg, Berend and Murphy, Kevin and Wood, Frank},
  journal={ICML},
  year={2025}
}

@inproceedings{somepalli2023diffusion,
  title={Diffusion art or digital forgery? investigating data replication in diffusion models},
  author={Somepalli, Gowthami and Singla, Vasu and Goldblum, Micah and Geiping, Jonas and Goldstein, Tom},
  booktitle={Proceedings of the IEEE/CVF conference on computer vision and pattern recognition},
  pages={6048--6058},
  year={2023}
}

@inproceedings{Carlini2023,
  title={Extracting training data from diffusion models},
  author={Carlini, Nicolas and Hayes, Jamie and Nasr, Milad and Jagielski, Matthew and Sehwag, Vikash and Tramer, Florian and Balle, Borja and Ippolito, Daphne and Wallace, Eric},
  booktitle={32nd USENIX Security Symposium (USENIX Security 23)},
  pages={5253--5270},
  year={2023}
}

@article{Gu2025memorizationdiffusion,
  title={On memorization in diffusion models},
  author={Gu, Xiangming and Du, Chao and Pang, Tianyu and Li, Chongxuan and Lin, Min and Wang, Ye},
  journal={TMLR},
  year={2025}
}

@article{Gong2022diffusiontext,
  title={Diffuseq: Sequence to sequence text generation with diffusion models},
  author={Gong, Shansan and Li, Mukai and Feng, Jiangtao and Wu, Zhiyong and Kong, LingPeng},
  journal={ICLR},
  year={2023}
}

@article{Xu2025energytext,
  title={Energy-based diffusion language models for text generation},
  author={Xu, Minkai and Geffner, Tomas and Kreis, Karsten and Nie, Weili and Xu, Yilun and Leskovec, Jure and Ermon, Stefano and Vahdat, Arash},
  journal={ICLR},
  year={2025}
}

@inproceedings{Ryzhakov2024explicit,
  title={Explicit Flow Matching: On The Theory of Flow Matching Algorithms with Applications},
  author={Ryzhakov, Gleb and Pavlova, Svetlana and Sevriugov, Egor and Oseledets, Ivan},
  booktitle={ICOMP},
  year={2024}
}

@article{Vastola2025generalization,
  title={Generalization through variance: how noise shapes inductive biases in diffusion models},
  author={J. J. Vastola},
  journal={ICLR},
  year={2025}
}

@article{kamb2024analytic,
  title={An analytic theory of creativity in convolutional diffusion models},
  author={Kamb, Mason and Ganguli, Surya},
  journal={ICML},
  year={2025}
}

@article{Krizhevsky2009,
  title={Learning multiple layers of features from tiny images},
  author={A. Krizhevsky and G. Hinton},
  year={2009},
  publisher={Toronto, ON, Canada}
}

@article{Ho2020,
  title={Denoising diffusion probabilistic models},
  author={J. Ho and A. Jain and P. Abbeel},
  journal={NeuRIPS},
  year={2020}
}

@misc{StablediffXL,
    author = {{Stability~AI}},
  howpublished={\url{https://stability.ai/stablediffusion}},
  note = {Accessed: 2023-09-09},
  version = {Stable Diffusion XL},
  year = 2023
}

@inproceedings{szegedy2016rethinking,
  title={Rethinking the inception architecture for computer vision},
  author={Szegedy, Christian and Vanhoucke, Vincent and Ioffe, Sergey and Shlens, Jon and Wojna, Zbigniew},
  booktitle={Proceedings of the IEEE conference on computer vision and pattern recognition},
  pages={2818--2826},
  year={2016}
}

@article{huang2021variational,
  title={A variational perspective on diffusion-based generative models and score matching},
  author={Huang, Chin-Wei and Lim, Jae Hyun and Courville, Aaron C},
  journal={NeurIPS},
  volume={34},
  pages={22863--22876},
  year={2021}
}

@article{oquab2023dinov2,
  title={Dinov2: Learning robust visual features without supervision},
  author={Oquab, Maxime and Darcet, Timoth{\'e}e and Moutakanni, Th{\'e}o and Vo, Huy and Szafraniec, Marc and Khalidov, Vasil and Fernandez, Pierre and Haziza, Daniel and Massa, Francisco and El-Nouby, Alaaeldin and others},
  journal={arXiv preprint arXiv:2304.07193},
  year={2023}
}

@article{heusel2017gans,
  title={Gans trained by a two time-scale update rule converge to a local nash equilibrium},
  author={Heusel, Martin and Ramsauer, Hubert and Unterthiner, Thomas and Nessler, Bernhard and Hochreiter, Sepp},
  journal={NeurIPS},
  volume={30},
  year={2017}
}

@article{gao2024flow,
  title={How Do Flow Matching Models Memorize and Generalize in Sample Data Subspaces?},
  author={Gao, Weiguo and Li, Ming},
  journal={arXiv preprint arXiv:2410.23594},
  year={2024}
}

@inproceedings{zhang2018unreasonable,
  title={The unreasonable effectiveness of deep features as a perceptual metric},
  author={Zhang, Richard and Isola, Phillip and Efros, Alexei A and Shechtman, Eli and Wang, Oliver},
  booktitle={Proceedings of the IEEE conference on computer vision and pattern recognition},
  pages={586--595},
  year={2018}
}

@article{Scarvelis2023,
  title={Closed-form diffusion models},
  author={C. Scarvelis and H. S. Borde de Oc{\'a}riz and J. Solomon},
  journal={TMLR},
  year={2025}
}

@article{Biroli2024,
  title={Dynamical regimes of diffusion models},
  author={G. Biroli and T. Bonnaire and V.  de Bortoli and M. M{\'e}zard},
  journal={Nature Communications},
  volume={15},
  number={1},
  pages={9957},
  year={2024},
  publisher={Nature Publishing Group UK London}
}

@article{Jiralerspong2023fld,
  title={Feature likelihood divergence: evaluating the generalization of generative models using samples},
  author={Jiralerspong, Marco and Bose, Joey and Gemp, Ian and Qin, Chongli and Bachrach, Yoram and Gidel, Gauthier},
  journal={{NeurIPS}},
  year={2023}
}

@article{stein2023evalgenmodels,
  title={Exposing flaws of generative model evaluation metrics and their unfair treatment of diffusion models},
  author={Stein, George and Cresswell, Jesse and Hosseinzadeh, Rasa and Sui, Yi and Ross, Brendan and Villecroze, Valentin and Liu, Zhaoyan and Caterini, Anthony L and Taylor, Eric and Loaiza-Ganem, Gabriel},
  journal={{NeurIPS}},
  volume={36},
  pages={3732--3784},
  year={2023}
}

@article{kadkhodaie2024generalization,
  title={Generalization in diffusion models arises from geometry-adaptive harmonic representations},
  author={Kadkhodaie, Zahra and Guth, Florentin and Simoncelli, Eero P and Mallat, St{\'e}phane},
  journal={ICLR},
  year={2024}
}

@article{ross2025memorization,
  title={A geometric framework for understanding memorization in generative models},
  author={Ross, Brendan Leigh and Kamkari, Hamidreza and Wu, Tongzi and Hosseinzadeh, Rasa and Liu, Zhaoyan and Stein, George and Cresswell, Jesse C and Loaiza-Ganem, Gabriel},
  journal={ICLR},
  year={2025}
}

@inproceedings{sohlDickstein2015diffusion,
  title={Deep unsupervised learning using nonequilibrium thermodynamics},
  author={Sohl-Dickstein, Jascha and Weiss, Eric and Maheswaranathan, Niru and Ganguli, Surya},
  booktitle={ICML},
  year={2015},
}

@article{martin2024pnpflow,
  title={PnP-Flow: Plug-and-play image restoration with flow matching},
  author={Martin, S{\'e}gol{\`e}ne and Gagneux, Anne and Hagemann, Paul and Steidl, Gabriele},
  journal={ICLR},
  year={2025}
}

@article{pooladian2023multisample,
  title={Multisample flow matching: Straightening flows with minibatch couplings},
  author={Pooladian, Aram-Alexandre and Ben-Hamu, Heli and Domingo-Enrich, Carles and Amos, Brandon and Lipman, Yaron and Chen, Ricky TQ},
  journal={ICML},
  year={2023}
}

@article{Lipman2024guide,
  title={Flow matching guide and code},
  author={Lipman, Yaron and Havasi, Marton and Holderrieth, Peter and Shaul, Neta and Le, Matt and Karrer, Brian and Chen, Ricky TQ and Lopez-Paz, David and Ben-Hamu, Heli and Gat, Itai},
  journal={arXiv preprint arXiv:2412.06264},
  year={2024}
}

@article{albergo2023stochasticinterpolant,
  title={Building normalizing flows with stochastic interpolants},
  author={Albergo, Michael S and Vanden-Eijnden, Eric},
  journal={ICLR},
  year={2023}
}

@article{lipman2023flow,
  title={Flow matching for generative modeling},
  author={Lipman, Yaron and Chen, Ricky TQ and Ben-Hamu, Heli and Nickel, Maximilian and Le, Matt},
  journal={ICLR},
  year={2023}
}

@inproceedings{Gagneux2025visual_flow_matching,
  title={A Visual Dive into Conditional Flow Matching},
  author={Gagneux, Anne and Martin, S{\'e}gol{\`e}ne and Emonet, R{\'e}mi and Bertrand, Quentin and Massias, Mathurin},
  booktitle={The Fourth Blogpost Track at ICLR},
  year={2025}
}

@article{song2020diffusion,
  title={Score-based generative modeling through stochastic differential equations},
  author={Song, Yang and Sohl-Dickstein, Jascha and Kingma, Diederik P and Kumar, Abhishek and Ermon, Stefano and Poole, Ben},
  journal={ICLR},
  year={2021}
}

@article{liu2023rectifiedflow,
  title={Flow straight and fast: Learning to generate and transfer data with rectified flow},
  author={Liu, Xingchao and Gong, Chengyue and Liu, Qiang},
  journal={ICLR},
  year={2023}
}

@article{videoworldsimulators2024,
  title={Video generation models as world simulators},
  author={Tim Brooks and Bill Peebles and Connor Holmes and Will DePue and Yufei Guo and Li Jing and David Schnurr and Joe Taylor and Troy Luhman and Eric Luhman and Clarence Ng and Ricky Wang and Aditya Ramesh},
  year={2024},
  url={https://openai.com/research/video-generation-models-as-world-simulators},
}

@article{Borsos2023,
  title={Audiolm: a language modeling approach to audio generation},
  author={Borsos, Zal{\'a}n and Marinier, Rapha{\"e}l and Vincent, Damien and Kharitonov, Eugene and Pietquin, Olivier and Sharifi, Matt and Roblek, Dominik and Teboul, Olivier and Grangier, David and Tagliasacchi, Marco and others},
  journal={IEEE/ACM Transactions on Audio, Speech, and Language Processing},
  year={2023},
  publisher={IEEE}
}

@inproceedings{Villegas2022,
  title={Phenaki: Variable length video generation from open domain textual descriptions},
  author={Villegas, Ruben and Babaeizadeh, Mohammad and Kindermans, Pieter-Jan and Moraldo, Hernan and Zhang, Han and Saffar, Mohammad Taghi and Castro, Santiago and Kunze, Julius and Erhan, Dumitru},
  booktitle={ICLR},
  year={2022}
}

@inproceedings{parmar2021cleanfid,
  title={On Aliased Resizing and Surprising Subtleties in GAN Evaluation},
  author={Parmar, Gaurav and Zhang, Richard and Zhu, Jun-Yan},
  booktitle={CVPR},
  year={2022}
}

@article{diffusion_fm_blogpost,
  title={Diffusion Meets Flow Matching: Two Sides of the Same Coin},
author = {R. Gao and E. Hoogeboom and J. Heek and V. de Bortoli and K. Murphy and T. Salimans},
  journal = {ICLR Blogpost},
  year = 2025
}

@inproceedings{tong2023improving,
  title={Improving and Generalizing Flow-Based Generative Models with Minibatch Optimal Transport},
  author={Alexander Tong and Nikolay Malkin and Guillaume Huguet and Yanlei Zhang and Jarrid Rector-Brooks and Kilian Fatras and Guy Wolf and Yoshua Bengio},
  booktitle={TMLR},
  year={2024},
  url={https://openreview.net/forum?id=CD9Snc73AW}
}


\newpage
\section*{NeurIPS Paper Checklist}

The checklist is designed to encourage best practices for responsible machine learning research, addressing issues of reproducibility, transparency, research ethics, and societal impact. Do not remove the checklist: {\bf The papers not including the checklist will be desk rejected.} The checklist should follow the references and follow the (optional) supplemental material.  The checklist does NOT count towards the page
limit.

Please read the checklist guidelines carefully for information on how to answer these questions. For each question in the checklist:
\begin{itemize}
    \item You should answer \answerYes{}, \answerNo{}, or \answerNA{}.
    \item \answerNA{} means either that the question is Not Applicable for that particular paper or the relevant information is Not Available.
    \item Please provide a short (1–2 sentence) justification right after your answer (even for NA).
\end{itemize}

{\bf The checklist answers are an integral part of your paper submission.} They are visible to the reviewers, area chairs, senior area chairs, and ethics reviewers. You will be asked to also include it (after eventual revisions) with the final version of your paper, and its final version will be published with the paper.

The reviewers of your paper will be asked to use the checklist as one of the factors in their evaluation. While "\answerYes{}" is generally preferable to "\answerNo{}", it is perfectly acceptable to answer "\answerNo{}" provided a proper justification is given (e.g., "error bars are not reported because it would be too computationally expensive" or "we were unable to find the license for the dataset we used"). In general, answering "\answerNo{}" or "\answerNA{}" is not grounds for rejection. While the questions are phrased in a binary way, we acknowledge that the true answer is often more nuanced, so please just use your best judgment and write a justification to elaborate. All supporting evidence can appear either in the main paper or the supplemental material, provided in appendix. If you answer \answerYes{} to a question, in the justification please point to the section(s) where related material for the question can be found.

IMPORTANT, please:
\begin{itemize}
    \item {\bf Delete this instruction block, but keep the section heading ``NeurIPS Paper Checklist"},
    \item  {\bf Keep the checklist subsection headings, questions/answers and guidelines below.}
    \item {\bf Do not modify the questions and only use the provided macros for your answers}.
\end{itemize}


\begin{enumerate}

\item {\bf Claims}
    \item[] Question: Do the main claims made in the abstract and introduction accurately reflect the paper's contributions and scope?
    \item[] Answer: \answerYes{} 
    \item[] Justification: each claim of the abstract refers to a specific subsection of the paper, that provide empirical evidence of the claim.

\item {\bf Limitations}
    \item[] Question: Does the paper discuss the limitations of the work performed by the authors?
    \item[] Answer: \answerYes{}
    \item[] Justification: We do have a specific section for the limitation of our work

\item {\bf Theory assumptions and proofs}
    \item[] Question: For each theoretical result, does the paper provide the full set of assumptions and a complete (and correct) proof?
    \item[] Answer: \answerYes{} 
    \item[] Justification: all results are encapsulated in clearly defined statements, and proofs are provided in appendix.

    \item {\bf Experimental result reproducibility}
    \item[] Question: Does the paper fully disclose all the information needed to reproduce the main experimental results of the paper to the extent that it affects the main claims and/or conclusions of the paper (regardless of whether the code and data are provided or not)?
    \item[] Answer: \answerYes{} 
    \item[] Justification: We provided as many details as possible in order to reproduce the results, in particular, we refer to the public implementation we used, including the specific (default) parameters used.

\item {\bf Open access to data and code}
    \item[] Question: Does the paper provide open access to the data and code, with sufficient instructions to faithfully reproduce the main experimental results, as described in supplemental material?
    \item[] Answer: \answerYes{} 
    \item[] Justification: Code will be made available along with publication

\item {\bf Experimental setting/details}
    \item[] Question: Does the paper specify all the training and test details (e.g., data splits, hyperparameters, how they were chosen, type of optimizer, etc.) necessary to understand the results?
    \item[] Answer: \answerYes{} 
    \item[] Justification: We provide a specific appendix with the experimental details

\item {\bf Experiment statistical significance}
    \item[] Question: Does the paper report error bars suitably and correctly defined or other appropriate information about the statistical significance of the experiments?
    \item[] Answer: \answerYes{} 
    \item[] Justification: We do not report error bars, however, we do specify the number of samples used for the FID computation and highlight the strong weaknesses of the FID metric.

\item {\bf Experiments compute resources}
    \item[] Question: For each experiment, does the paper provide sufficient information on the computer resources (type of compute workers, memory, time of execution) needed to reproduce the experiments?
    \item[] Answer: \answerYes{} 
    \item[] Justification: we specified what type of GPU we used

\item {\bf Code of ethics}
    \item[] Question: Does the research conducted in the paper conform, in every respect, with the NeurIPS Code of Ethics \url{https://neurips.cc/public/EthicsGuidelines}?
    \item[] Answer: \answerYes{} 
    \item[] Justification:  \answerNA{}

\item {\bf Broader impacts}
    \item[] Question: Does the paper discuss both potential positive societal impacts and negative societal impacts of the work performed?
    \item[] Answer: \answerYes{} 
    \item[] Justification: there is a dedicated broader impact section

\item {\bf Safeguards}
    \item[] Question: Does the paper describe safeguards that have been put in place for responsible release of data or models that have a high risk for misuse (e.g., pretrained language models, image generators, or scraped datasets)?
    \item[] Answer: \answerNo{} 
    \item[] Justification: We work on standard image datasets

\item {\bf Licenses for existing assets}
    \item[] Question: Are the creators or original owners of assets (e.g., code, data, models), used in the paper, properly credited and are the license and terms of use explicitly mentioned and properly respected?
    \item[] Answer: \answerYes{} 
    \item[] Justification: We properly refer the \texttt{torchcfm} and \texttt{PnPflow} codebase.

\item {\bf New assets}
    \item[] Question: Are new assets introduced in the paper well documented and is the documentation provided alongside the assets?
    \item[] Answer: \answerNA{}  
    \item[] Justification: \answerNA{}

\item {\bf Crowdsourcing and research with human subjects}
    \item[] Question: For crowdsourcing experiments and research with human subjects, does the paper include the full text of instructions given to participants and screenshots, if applicable, as well as details about compensation (if any)?
    \item[] Answer: \answerNA{}  
    \item[] Justification: \answerNA{}

\item {\bf Institutional review board (IRB) approvals or equivalent for research with human subjects}
    \item[] Question: Does the paper describe potential risks incurred by study participants, whether such risks were disclosed to the subjects, and whether Institutional Review Board (IRB) approvals (or an equivalent approval/review based on the requirements of your country or institution) were obtained?
    \item[] Answer: \answerNA{} 
    \item[] Justification: \answerNA{}

\item {\bf Declaration of LLM usage}
    \item[] Question: Does the paper describe the usage of LLMs if it is an important, original, or non-standard component of the core methods in this research? Note that if the LLM is used only for writing, editing, or formatting purposes and does not impact the core methodology, scientific rigorousness, or originality of the research, declaration is not required.
    \item[] Answer: \answerNo{} 
    \item[] Justification: LLMs were only used for grammatical purposes.

\end{enumerate}

\newpage
\appendix

\section{Proofs of \Cref{sec:cfmrecalls}}
\begin{align}\label{eq:general_closed_form}
    \utothat(x, t) = \sum_{i=1}^n \ucond(x, z=\data{i}, t) \cdot \frac{p(x | z=\data{i}, t)}{\sum_{i'=1}^n p(x | z=\data{i'}, t)}
    \enspace.
\end{align}
\subsection{Proof of \Cref{prop_closed_form_velocity}}
\label{app:pf_closed_form_velocity}
\begin{proof}
$\bullet$ In the case where $z \sim \pdatahat$, conditional probability writes
    \begin{align}
        p(z=\data{i} | x, t)
        &= \frac{p(x, t, z=\data{i})}{p(x, t)}
        \\
        &= \frac{p(x | t, z=\data{i}) p(t, z=\data{i}) }{p(x, t)}
        \\
        &= \frac{p(x | t, z=\data{i}) p(t, z=\data{i}) }{\sum_{i'=1}^n p(x, t, z=\data{i'})}
        \\
        &= \frac{p(x | t, z=\data{i}) p(t) \overbrace{p(z=\data{i})}^{\frac{1}{n}} }{\sum_{i'=1}^n p(x | t, z=\data{i'})  p(t) \underbrace{p(z=\data{i'})}_{\frac{1}{n}} }
        \\
        &= \frac{p(x | t, z=\data{i}) }{\sum_{i'=1}^n p(x | t, z=\data{i'})}
        \label{eq_conditional_prob}
        \enspace.
    \end{align}
    Pluging \Cref{eq_conditional_prob} in \Cref{eq_inversion_formula} yields the closed-formed formula for the velocity field:
    \begin{align}
        \utot(x, t ) &=
        \sum_{i=1}^n \ucond(x, t, z=\data{i}) p(z=\data{i} | x, t)
        \\
        &=
        \sum_{i=1}^n \ucond(x, t, z=\data{i}) \frac{p(x | t, z=\data{i}) }{\sum_{i'=1}^n p(x | t, z=\data{i'})}
        \enspace.
    \end{align}
    which proves \Cref{eq:general_closed_form}; using that $x | t, z = \data{i} \sim \cN(t \data{i}, (1 -t)^2 \Id)$ and $\ucond(x, t, z = \data{i}) = \frac{\data{i} - x}{1 - t}$ yields \Cref{eq:closed_form}.

$\bullet$ For the case $z \sim p_0 \times \pdatahat$,
\begin{align}
    \utothat(x, t)
    &:= \int_z \ucond(x, t, z) p(z | x, t) \dz \\
    &= \int_z  \ucond(x, t, z) \frac{p(x, z, t)}{p(x, t)} \dz \\
    &= \int_z \ucond(x, t, z) \frac{p(x| z, t) p(z) p(t)}{\int_{z'} p(x| t, z')p(t) p(z') \dz'} \dz \\
    &= \int_z \ucond(x, t, z) \frac{p(x| z, t) p(z)}{\int_{z'} p(x|t, z') p(z') \dz'} \dz
\end{align}
Since $z \sim p_0 \times \pdatahat$, the denominator is equal to:

\begin{align}
    \label{eq:velo1}
    \int_{z'} p(x| t, z') p(z') \dz'
    &= \frac1n \int_{x_0} \sum_{i=1}^n  \delta_{x}((1 - t) x_0 + t \data{i}) \frac{1}{\sqrt{(2\pi)^d}} \exp(-\frac{1}{2} x_0^2) \dx_0  \\
        &= \frac1n \int_{y} \sum_{i=1}^n  \delta_{x}(y) \frac{1}{\sqrt{(2\pi)^d}} \exp(-\frac{1}{2(1-t)^2} \Vert y - t\data{i} \Vert ^2) \frac{1}{(1-t)^d} \mathrm{d}y  \quad \quad (y = (1 - t) x_0 + t \data{i})\\
    &=\frac1n \sum_{i=1}^n \frac{1}{\sqrt{(2\pi (1 -t)^2)^d}} \exp\left(-\frac{1}{2 (1-t)^2} \Vert x - t\data{i} \Vert^2\right)
\end{align}

Likewise, the numerator equals:
\begin{align}
\label{eq:velo2}
    \int_z \ucond(x, t, z) p(x| z, t) p(z)\dz
    & =
    \int_{x_0} \frac1n \sum_{i=1}^n  (\data{i} - x_0) \delta_{x}((1-t)x_0 + t x^{(i)}) \frac{1}{\sqrt{(2\pi)^d}} \exp(-\frac{1}{2} \Vert x_0 \Vert^2) \dx_0  \\
    & = \frac1n  \sum_{i=1}^n \int_{y} \frac{\data{i} - y}{1-t} \delta_{x}(y) \frac{1}{\sqrt{(2\pi(1-t)^2)^d}} \exp(-\frac{1}{2(1-t)^2} \Vert y - tx^{(i)} \Vert^2) \mathrm{d} y  \\
    &= \sum_{i=1}^n \frac{\data{i} - x}{1 - t}  \frac{1}{\sqrt{(2\pi(1-t)^2)^d}} \exp\left(-\frac{1}{2 (1-t)^2} \Vert x - t\data{i})\Vert^2\right)  \label{eq:velo2}
\end{align}
Taking the ratio of \Cref{eq:velo1,eq:velo2} concludes the proof.
\end{proof}



\section{Additional details and comments on empirical flow matching}
\label{app_efm}
First, recalls on the optimal velocity (\Cref{eq:closed_form}) and the empirical flow matching loss (\Cref{eq:l_efm,eq:l_efm_softmax}) are provided in \Cref{app_sub_recalls}.
The unbiasedness of the estimator is presented in \Cref{app_sub_theoretical_res}, and its proof is in \Cref{app_sub_proof_efm}.

\subsection{Recalls}
\label{app_sub_recalls}
The closed-form formula of the "optimal" velocity field is:
\begin{align} \tag{\ref{eq:closed_form}}
        \utothat(x, t)
    & =
    \sum_{l=1}^{n}
    \frac{\data{l} - x}{1-t}
    \cdot
    \left[
        \softmax \left( \left(-\frac{\normin{x - t \data{k} }^2}{2(1 -t)^2} \right)_{k = 1, \dots, n} \right)
    \right]_l
    \enspace.
\end{align}

The proposed loss uses mini-batches of size $M$ (instead of all $n$ training points) to build an estimator $\utothatemp$ of $\utothat$:

\begin{empheq}[box=\fcolorbox{blue!40!black!10}{green!05}]{align} \tag{\ref{eq:l_efm}}
    \mathcal{L}_{\mathrm{EFM}}(\theta)
    & =
    \mathbb{E}_{
        \substack{
        t \sim \mathcal{U}([0, 1]) \\
        x_0 \sim p_0 \\
        x_1 \sim \pdatahat \\
        x_t = (1-t) x_0 + t x_1 \\
        \batch{1} := x_1 \; ;\; \batch{2}, \dots, \batch{M} \sim \pdatahat
        }}
    \Vert u_\theta(x_t, t) -
    \utothatemp(x_t, t) \Vert^2
    \enspace,
    \\
    \text{ with} \nonumber\\
    \utothatemp(x_t, t)
    & =
    \sum_{j=1}^{M}
    \frac{\batch{j} - x_t}{1-t}
    \cdot
    \left[
        \softmax \left( \left(-\frac{\normin{x_t - t b^{(k)} }^2}{2(1 -t)^2} \right)_{k = 1, \dots, M} \right)
    \right]_j
    \enspace.
    \tag{\ref{eq:l_efm_softmax}}
\end{empheq}

Crucially, in \Cref{eq:l_efm}  the sample $\batch{1}$ depends on $x_t$ and is reused in the estimate $\utothatemp$.
This important detail yields an unbiased estimator of $\utothat$.

\subsection{Theoretical properties of the proposed estimator}
\label{app_sub_theoretical_res}

First, we discuss below the relation between \Cref{app_prop_efm} and the sampling literature.

\xhdr{Links with importance sampling}
The estimator $\utothat$ in \Cref{eq:closed_form} can be seen as a form of \textit{importance sampling} (see \citealt[Chap. 3]{robert1999montecarlobook} for an in-depth reference).
In a nutshell, importance sampling is a way to estimate an expectation when one cannot easily sample from the random variable it depends on.
More precisely, in the ideal case $z \sim \pdata$ (as opposed to $z \sim \pdatahat$), the velocity field formula is the following
\begin{align}
    \utot(x_t, t)
    &=
    \mathbb{E}_{z | x_t, t} \left[\ucond(x_t, z, t) \right] \\
   &=
    \int_{z} \ucond(x_t, z, t)  p(z | x_t, t)  \diff z
    \enspace.
\end{align}
When $z \sim \pdata$, it is difficult to sample from $z | x_t, t$, but the latter equation can be rewritten as
\begin{align}
    \utot(x_t, t)
   &=
    \int_{z} \ucond(x_t, z, t) \frac{ p(z | x_t, t)}{p(z)} p(z)  \diff z
\end{align}
and one can easily sample from $z \sim \pdatahat$ using the empirical data distribution $\data{1}, \dots, \data{n}$
\begin{align}
    \utot(x_t, t)
    &
    \approx
    \frac{1}{n} \sum_{i=1}^n \ucond(x_t, \data{i}, t) \frac{ p(z=\data{i} | x_t, t)}{p(\data{i})}
    \\
    & \phantom{\approx} =
    \sum_{i=1}^n \ucond(x_t, \data{i}, t) p(z=\data{i} | x_t, t)
    \\
    & \phantom{\approx} :=
    \utothat(x_t, t) \enspace.
\end{align}


\subsection{Proof of \Cref{app_prop_efm}}
\label{app_sub_proof_efm}
We first recall \Cref{app_efm}, which we prove in this section.
\propefm*

\begin{proof}[Proof of Item~\eqref{app_prop_minimizer}.]
    With no constraints on $\utheta$, the empirical flow matching loss writes:
    \begin{align}
        &
        \;
        \mathbb{E}_{
            \substack{
            t \sim \mathcal{U}([0, 1]) \\
            x_1 \sim \pdatahat \\
            x_t = (1-t) x_0 + t x_1 \\
            \batch{1} := x_1 \; ;\; \batch{2}, \dots, \batch{M} \sim \pdatahat
            }}
        \Vert u_\theta(x_t, t) -
        \utothatemp(x_t, t) \Vert^2
        \enspace,
        \\
        =\;
        &
        \mathbb{E}_{
            \substack{
            t \sim \mathcal{U}([0, 1]) \\
            x_t \sim p_t \\
            }}
        \mathbb{E}_{
            \substack{
            \batch{1} \sim \pdatahat(\cdot | x_t, t) \\
            \batch{2}, \dots, \batch{M} | x_t, t
            }}
        \Vert u_\theta(x_t, t) -
        \utothatemp(x_t, t) \Vert^2
        \enspace,
        \\
        =\;
        &
        \mathbb{E}_{
            \substack{
            t \sim \mathcal{U}([0, 1]) \\
            x_t \sim p_t \\
            }}
        \mathbb{E}_{
            \substack{
            \batch{1} := \pdatahat(\cdot | x_t, t) \\
            \batch{2}, \dots, \batch{M} \sim \pdatahat
            }}
        \Vert u_\theta(x_t, t) -
        \utothatemp(x_t, t) \Vert^2
        \enspace \text{because $\batch{2}, \dots, \batch{M} \indep x_t, t$} \enspace,
        \label{app_eq_toward_minimizer}
    \end{align}
which is minimized when for all $x_t, t$
\begin{align}
    \utheta(x_t, t)
    =
    \mathbb{E}_{
        \substack{
        \batch{1} \sim \pdatahat(\cdot | x_t, t) \\
        \batch{2}, \dots, \batch{M} \sim \pdatahat
        }}
        \left[\utothatemp(x_t, t)\right]
        \enspace.
\end{align}

\end{proof}
\begin{proof}[Proof of Item~\eqref{app_prop_unbiased}.]
    The minimizer for a given $(x_t,t)$, removing these elements from the notation for conciseness and abstraction, is a weighted mean:

    \newcommand{\uc}[1]{u^{(#1)}}
    \newcommand{\w}[1]{w^{(#1)}}
    \begin{align}
        \utothat(x_t,t) = \utothat & = \sum_{l=1}^n \w{l} \uc{l}
        \enspace,
        \text{ with }
        \\
        \w{l}  &= \pdatahat(z=\data{l}|t,x_t)
        \enspace , \;
        \sum_{l=1}^n \w{l} = 1 \label{eq_w_softmax}
        \\
        \uc{l} &= \ucond(x_t, \data{l}, t)
    \end{align}

    We express a mini-batch as an $M$-valued vector of indices, $\bm{i} \in \llbracket 1, n \rrbracket^{M}$.
    The mini-batch estimate from~\Cref{eq:l_efm}, considering the definition of the $\softmax$, can be expressed as a mini-batch weighted-mean:

    \begin{align}
        \utothatemp(\bf{i}) &= \frac{\sum_{j=1}^{M} \w{\bf{i}_j} \uc{\bf{i}_j}}{\sum_{j=1}^{M} \w{\bf{i}_j}}
    \end{align}
    The categorical distribution over $\llbracket 1, n \rrbracket$ with probabilities following the weights $w$ in \eqref{eq_w_softmax} is denoted $\operatorname{Cat}(w)$ and the uniform distribution, \ie $\operatorname{Cat}(\mathbf{1}/n))$, is denoted $\operatorname{Unif}$.

    The main result of the following is that, in expectation over the biased-mini-batches, \textbf{where the first point is drawn according to $w$} and the $M-1$ other points are drawn uniformly, the mini-batch weighted-mean is an unbiased estimate of the $w$-weighted-mean $\utothat$.

    \begin{align}
        \mathbb{E} \left[\utothatemp(\bf{i}) \right]
        &:= \mathbb{E}_{{\bf{i}_1} \sim \operatorname{Cat}(w)} \mathbb{E}_{\bm{i}_2,...,\bm{i}_M \sim \operatorname{Unif}} \left[\utothatemp(\bf{i})\right] \\
        &= \sum_{\bm{i}_1=1}^{n} \w{\bm{i}_1}  \mathbb{E}_{\bm{i}_2,...,\bm{i}_M \sim \operatorname{Unif}} \left[\utothatemp(\bf{i})\right] \\
        &= \sum_{\bm{i}_1=1}^{n} \mathbb{E}_{\bm{i}_2,...,\bm{i}_M \sim \operatorname{Unif}} \left[\w{\bm{i}_1} \utothatemp(\bf{i})\right] \\
        &= n \sum_{\bm{i}_1=1}^{n} \frac{1}{n} \mathbb{E}_{\bm{i}_2,...,\bm{i}_M \sim \operatorname{Unif}} \left[\w{\bm{i}_1} \utothatemp(\bf{i})\right] \\
        &= n \; \mathbb{E}_{\bm{i}_1 \sim \operatorname{Unif}} \mathbb{E}_{\bm{i}_2,...,\bm{i}_M \sim \operatorname{Unif}} \left[\w{\bm{i}_1} \utothatemp(\bf{i})\right] \\
        &= n \; \mathbb{E}_{\bm{i}_1,...,\bm{i}_M \sim \operatorname{Unif}} \left[\w{\bm{i}_1} \utothatemp(\bf{i})\right]
        \label{eq_expectation_refactored}
    \end{align}

    The expression in \Cref{eq_expectation_refactored} is invariant with respect to order of the indices $\bm{i}_1, \dots, \bm{i}_M$: the indices in expectation in \Cref{eq_expectation_refactored} can be exchanged, and one thus has
    \begin{align}
        \textcolor{darkgreen}{\forall k \in \llbracket 1, M \rrbracket} , \;
        \mathbb{E} \left[ \utothatemp({\bf{i}}) \right] = n \; \mathbb{E}_{\bm{i}_1,...,\bm{i}_M \sim \operatorname{Unif}} \left[\w{{\textcolor{darkgreen}{\bm{i}_k}}} \utothatemp(\bf{i})\right]
        \label{eq_expectation_refactored_k} \enspace.
    \end{align}

    Averaging \Cref{eq_expectation_refactored_k} over the indices $k \in \llbracket 1, M \rrbracket$ yields the desired result

    \begin{align}
        \frac{1}{M} \sum_{k=1}^M \mathbb{E} \utothatemp(\bf{i}) &= \frac{1}{M} \sum_{k=1}^M n \; \mathbb{E}_{\bm{i}_1,...,\bm{i}_M \sim \operatorname{Unif}} \left[\w{\bm{i}_k} \utothatemp({\bf{i}})\right] \\
        \mathbb{E} \utothatemp(\bf{i}) &= \frac{1}{M} n \; \mathbb{E}_{\bm{i}_1,...,\bm{i}_M \sim \operatorname{Unif}} \left[\sum_{k=1}^M \w{\bm{i}_k} \utothatemp({\bf{i}})\right] \\
        &= \frac{1}{M} n \; \mathbb{E}_{\bm{i}_1,...,\bm{i}_M \sim \operatorname{Unif}} \left[\sum_{k=1}^M \w{\bm{i}_k} \frac{\sum_{j=1}^{M} \w{\bf{i}_j} \uc{\bf{i}_j}}{\sum_{j=1}^{M} \w{\bf{i}_j}} \right] \\
        &= \frac{1}{M} n \; \mathbb{E}_{\bm{i}_1,...,\bm{i}_M \sim \operatorname{Unif}} \left[\left(\sum_{k=1}^M \w{\bm{i}_k}\right) \frac{\sum_{j=1}^{M} \w{\bf{i}_j} \uc{\bf{i}_j}}{\left(\sum_{j=1}^{M} \w{\bf{i}_j}\right)} \right] \\
        &= \frac{1}{M} n \; \mathbb{E}_{\bm{i}_1,...,\bm{i}_M \sim \operatorname{Unif}} \left[ \sum_{j=1}^{M} \w{\bf{i}_j} \uc{\bf{i}_j} \right] \\
        &= \frac{1}{M} n \; \sum_{j=1}^{M} \mathbb{E}_{\bm{i}_1,...,\bm{i}_M \sim \operatorname{Unif}} \left[  \w{\bf{i}_j} \uc{\bf{i}_j} \right] \\
        &= \frac{1}{M} n \; \sum_{j=1}^{M} \mathbb{E}_{\bm{i}_j \sim \operatorname{Unif}} \left[  \w{\bf{i}_j} \uc{\bf{i}_j} \right] \\
        &= \frac{1}{M} n \; M \; \mathbb{E}_{l \sim \operatorname{Unif}} \left[  \w{l} \uc{l} \right] \\
        &= n \; \mathbb{E}_{l \sim \operatorname{Unif}} \left[  \w{l} \uc{l} \right] \\
        &= n \; \sum_{l=1}^n \frac{1}{n} \left[  \w{l} \uc{l} \right] \\
        &= \sum_{l=1}^n \left[  \w{l} \uc{l} \right] \\
        &= \utothat
    \end{align}
\end{proof}

\begin{proof}[Proof of Item~\eqref{app_prop_variance}.]
    Using the same ideas as for Item~\eqref{app_prop_unbiased}, one has
    \begin{align}
        &\mathbb{E}_{\data{1} \sim \pdatahat(\cdot | x_t, t ) \; ; \;\batch{2}, \dots, \batch{M} \sim \pdatahat }
        \left[\utothatemp(x_t, t)^2\right]
        \\
        &=
        n \mathbb{E}_{\bm{i}_1, \dots, \bm{i}_M \sim \mathrm{Unif}}
        \left [
            w^{(\bm{i}_1)} \utothatemp(\bm{i})^2
            \right ]
        \\
        &=
        n \mathbb{E}_{\bm{i}_1, \dots, \bm{i}_M \sim \mathrm{Unif}}
        \left [
            w^{(\bm{i}_k)} \utothatemp(\bm{i})^2
            \right ], \forall k \in \llbracket 1, M\rrbracket
        \\
        &=
        n \frac{1}{M} \mathbb{E}_{\bm{i}_1, \dots, \bm{i}_M \sim \mathrm{Unif}}
        \left [
            \sum_{k=1}^M w^{(\bm{i}_k)} \utothatemp(\bm{i})^2
        \right ]
        \\
        &=
        n \frac{1}{M} \mathbb{E}_{\bm{i}_1, \dots, \bm{i}_M \sim \mathrm{Unif}}
        \left [
            \sum_{k=1}^M w^{(\bm{i}_k)}
            \left (\frac{\sum_{j=1}^{M}
            w^{(\bm{i}_j)} u^{(\bm{i}_j)}
            }{
                \sum_{j=1}^{M} w^{(\bm{i}_j)}
            } \right )^2
        \right ]
        \\
        & \leq
        n \frac{1}{M} \mathbb{E}_{\bm{i}_1, \dots, \bm{i}_M \sim \mathrm{Unif}}
        \left [
            \sum_{k=1}^M w^{(\bm{i}_k)}
                \frac{
            \sum_{j=1}^{M}
            w^{(\bm{i}_j)} (u^{(\bm{i}_j)})^2
            }{
                \sum_{j=1}^{M} w^{(\bm{i}_j)}
            }
        \right ]
        \; \text{by convexity of }  x \mapsto x^2
        \\
        &\phantom{\leq}=
        n \frac{1}{M} \mathbb{E}_{\bm{i}_1, \dots, \bm{i}_M \sim \mathrm{Unif}}
        \left [
            \left( \sum_{k=1}^M w^{(\bm{i}_k)} \right)
                \frac{
            \sum_{j=1}^{M}
            w^{(\bm{i}_j)} (u^{(\bm{i}_j)})^2
            }{
                \sum_{j=1}^{M} w^{(\bm{i}_j)}
            }
        \right ]
        \\
        &\phantom{\leq}=
        n \frac{1}{M} \mathbb{E}_{\bm{i}_1, \dots, \bm{i}_M \sim \mathrm{Unif}}
        \left [
            \sum_{j=1}^{M}
            w^{(\bm{i}_j)} (u^{(\bm{i}_j)})^2
        \right ]
        \\
        &\phantom{\leq}=
        \mathbb{E}_{\bm{i}_1 \sim \mathrm{Unif}}
            \left[w^{(\bm{i}_1)} (u^{(\bm{i}_1)})^2\right]
        \\
        &\phantom{\leq}=
        \mathbb{E}_{l \sim \mathrm{Unif}}
            \left[w^{(l)} (u^{(l)})^2\right]
            \enspace.
    \end{align}
    Hence
    \begin{align}
        \mathbb{E}_{\data{1} \sim \pdatahat(\cdot | x_t, t ) \; ; \;\batch{2}, \dots, \batch{M} \sim \pdatahat }
        \left[\utothatemp(x_t, t)^2\right]
        -
        (\utothat)^2
        \leq
        \mathbb{E}_{l \sim \mathrm{Unif}}
        \left[w^{(l)} (u^{(l)})^2\right]
        -
        (\utothat)^2
        \enspace,
    \end{align}
    which is exactly
    \begin{align}
        \Var_{\data{1} \sim \pdatahat(\cdot | x_t, t ) \; ; \;\batch{2}, \dots, \batch{M} \sim \pdatahat }
        \left[\utothatemp(x_t, t)\right]
        \leq
        \Var_{\data{1} \sim \pdatahat(\cdot | x_t, t )}
        \left[\ucond(x_t, \data{1}, t)\right]
        \enspace.
    \end{align}
\end{proof}


\section{Additional experiments}

We present below the results for the MNIST dataset. The conclusions atre the same as for the CIFAR-10 and CelebA $64 \times 64$: regressing against a more deterministic velocity field does not hurt generalization. On the contrary, generalization (\ie lower test FID) appears earlier during training.

For this experiment, we used the Unet with attention and timestep embedding from \texttt{torchcfm} library, with the Adam optimizer and all the default parameters. We used a pretrained classifier with $99\%$ accuracy on MNIST ($90\%$ on FMNIST) as a lower-dimensional embedding of size $128$ to compute the FID between the test set and the generated set.

\begin{table}[H]
    \centering
    \setlength{\tabcolsep}{6pt} 
    \renewcommand{\arraystretch}{1.2} 
    \begin{tabular}{lccccccccc}
    \toprule
    \textbf{Method} & \textbf{Ep. 1} & \textbf{Ep. 2} & \textbf{Ep. 3} & \textbf{Ep. 4} & \textbf{Ep. 5} & \textbf{Ep. 10} & \textbf{Ep. 15} & \textbf{Ep. 20} & \textbf{Ep. 25} \\
    \midrule
    CFM (EFM, M=1)   & 378.00 & 181.25 & 67.88 & 29.44 & 15.30 & 4.20 & 3.08 & 2.51 & 2.28 \\
    EFM, M=128        & 370.64 & 168.58 & 60.52 & 25.52 & 13.44 & 3.79 & 2.70 & 2.35 & 2.10 \\
    EFM, M=256        & 370.94 & 169.71 & 61.88 & 25.73 & 13.48 & 3.73 & 2.76 & 2.33 & 2.08 \\
    EFM, M=1024       & 369.72 & 168.43 & 60.28 & 24.24 & 12.26 & 3.30 & 2.67 & 2.17 & 1.84 \\
    \bottomrule
    \addlinespace[0.75em]
    \end{tabular}
    \caption{\textbf{FID FMNIST}. FID scores across training epochs for conditional flow matching and empirical flow matching for multiple values of the number of samples $M$ used to estimate the closed-form $\utothat$.}
    \label{tab:fid_mnist}
\end{table}

\begin{table}[H]
    \centering
    \begin{tabular}{lccccc}
    \toprule
    \textbf{Method} & \textbf{FID Ep. 5} & \textbf{FID Ep. 10} & \textbf{FID Ep. 50} & \textbf{FID Ep. 100} & \textbf{FID Ep. 200} \\
    \midrule
    CFM (EFM, M=1)   & 253.56 & 48.67 & 25.36 & 21.35 & 19.67 \\
    EFM, M=128        & 206.27 & 44.08 & 23.39 & 19.63 & 17.72 \\
    EFM, M=256        & 202.62 & 45.06 & 22.16 & 20.08 & 17.74 \\
    EFM, M=512        & 194.66 & 44.19 & 22.10 & 18.93 & 16.85 \\
    \bottomrule
    \addlinespace[0.5em]
    \end{tabular}
    \caption{\textbf{FID FMNIST}. FID scores across training epochs for conditional flow matching and empirical flow matching for multiple values of the number of samples $M$ used to estimate the closed-form $\utothat$.}
    \label{tab:fid_fmnist}
\end{table}


\section{Experiments details}
\label{app:expes_details}
For all the experiment we used all the same learning hyperparameters, the default ones form \citet{tong2023improving}.
The hyperparameter values are summarized in \Cref{app_tab_hyperparam_cifar}.
The details specific to each figure are described in \Cref{app_expes_details_hist_cosine,app_expes_dist_ustar,app_expes_umix,app_expes_fid_cifar}

\begin{table}[h]
    \centering
    \begin{tabular}{lcccc}
    \toprule
    $\#$ Channels  & Batch Size & Learning Rate & EMA Decay & Gradient Clipping \\
    \midrule
    128 & 128 & 0.0002 & 0.9999 & 1  \\
    \bottomrule
    \end{tabular}
    \caption{Learning hyperparameters for all the CIFAR-10 and CelebA $64$ experiments.}
    \label{app_tab_hyperparam_cifar}
\end{table}

\subsection{Compute time}
Given that regressing against an estimate of the closed-form, EFM, seems to improve on CFM, one may wonder what is the additional cost induced by EFN. To alleviate the non-linearity of GPU computing (parallelism may cause some discontinuities in terms of costs), we ran an exhaustive set of timing experiments, varying the batch size and the EFM sample size. To summarize the measurements (numbers are given for an NVIDIA L4 GPU, on CIFAR-10), denoting $b$ the batch size and $e$ the EFM sample size, the cost follows $b \times ( 4.3ms + e \times 0.9\mu s)$. It can be also be seen as adding $\sim 2\%$ for every 100 EFM samples. Or, for instance with a batch size of 256, $1.1$ second will be due to the $256$-sample forward/backward, while the additional cost for EFM-$1000$ will be $230$ms (around $17\%$ of the cost) and for EFM-$128$ under $30$ms (under $3\%$).

\subsection{\Cref{fig:hist_cosine,fig:collapse_times}}
\label{app_expes_details_hist_cosine}
For \Cref{fig:hist_cosine} no deep learning is involved: the datasets 2-moons and CIFAR-10 are loaded.
Then, $256$ points from $p_0 \times \pdatahat$ are drawn, and one computes the mean of the cosine similarities between $\utothat((1-t) x_0 + t x_1, t)$ and $\ucond((1-t)  x_0 + t  x_1, z=x_1, t) = x_1 - x_0$, for each value of $t \in \{0, 1/100, 2/100, \dots, 99/100 \}$.

No deep learning either is involved in \Cref{fig:collapse_times}: the Imagenette dataset is loaded and spatially subsampled to resolution $\mathrm{dim}=8$, $\mathrm{dim}=16$, \dots, $\mathrm{dim}=256$, \ie with $d=" \cdot 8^2$, $d=3 \cdot 16^2$, \dots, $d=3 \cdot 256^2$.
Then, as for \Cref{fig:hist_cosine}, batches of $256$ points from $p_0$ and $\pdata$ are drawn, and one computes the percentage of cosine similarities between $\utothat((1-t) x_0 + t x_1, t)$ and $\ucond((1-t)  x_0 + t  x_1, z=x_1, t) = x_1 - x_0$, that are larger than $0.9$, for multiple time values $t$.

\subsection{\Cref{fig_dist_ustar}}
\label{app_expes_dist_ustar}
In \Cref{fig_dist_ustar}, networks are trained with a vanilla conditional flow matching, with the standard 34 million parameters U-Net for diffusion by \citet{nichol2021improved}, with default settings from the \texttt{torchfm} codebase \footnote{\url{https://github.com/atong01/conditional-flow-matching}} \citep{tong2023improving}.
Training uses the CFM loss.
For this specific experiment, \textbf{we removed the usual random flip transform}, for $\utothat$ to be simpler and easier to estimate by $\utheta$.
For each ``data'' subsampling of the dataset, we trained the model for $5 \cdot 10^4$ iterations, with a batch size of $128$,  \ie we trained the models for $128$ epochs.


\subsection{\Cref{fig:umix}}
\label{app_expes_umix}
In \Cref{fig:umix}, for each dataset (CIFAR-10 and CelebA $64 \times 64$), one network is trained using a vanilla conditional flow matching with the default parameters of \citet{tong2023improving} (the most important ones are recalled in \Cref{app_tab_hyperparam_cifar}).
Then images are generated first following the closed-form formula of the optimal velocity field $\utothat$ from $0$ to $\tau$. And then following the velocity field learned with a usual conditional flow matching $\utheta$ from $\tau$ to $1$.

\subsection{\Cref{fig_fid_cifar}}
\label{app_expes_fid_cifar}
For experiments involving training on CIFAR-10 (\Cref{fig_dist_ustar,fig:umix}), we rely on the standard 34 million parameters U-Net for diffusion by \citet{nichol2021improved}, with default settings from the \texttt{torchfm} codebase \citep{tong2023improving}.
For each algorithm, the networks are trained for $500$k iterations with batch size $128$, \ie $1280$ epochs.

For CelebA $64 \times 64$ (\Cref{fig:umix}), we rely on the training script of \texttt{pnpflow} library\footnote{\url{https://github.com/annegnx/PnP-Flow}} \citep{martin2024pnpflow}, which uses a U-Net from \citet{huang2021variational,Ho2020}.

\end{document}